\documentclass{article}

\usepackage{times}
\usepackage{graphicx} 
\usepackage{subfigure}

\usepackage{natbib}

\usepackage{algorithm}
\usepackage{algorithmic}

\usepackage{amsmath}		
\usepackage{amsfonts}		
\usepackage{amssymb}
\usepackage{bm}				
\usepackage{dsfont}
\usepackage{enumerate}		
\usepackage{mathtools}		
\usepackage{appendix}

\usepackage{booktabs}

\usepackage{amsthm}
\theoremstyle{plain}

\newtheorem{lemma}{Lemma}

\theoremstyle{remark}

\theoremstyle{definition}

\newcommand{\rms}[1]{\mathrm{#1}}
\newcommand{\bfs}[1]{\mathbf{#1}}
\renewcommand{\algorithmicrequire}{\textbf{Input:}}
\renewcommand{\algorithmicensure}{\textbf{Output:}}

\usepackage{graphicx}
\usepackage{wrapfig}
\usepackage{epstopdf}
\usepackage[font=small,labelfont=bf]{caption}

\newcommand{\topleftmatrix}{\bfs{E}}
\newcommand{\offdiagmatrix}{\bfs{F}}
\newcommand{\bottomrightmatrix}{\bfs{G}}

\usepackage{hyperref}

\usepackage[accepted]{icml2017}

\icmltitlerunning{Magnetic Hamiltonian Monte Carlo}

\begin{document}

	\twocolumn[
	\icmltitle{Magnetic Hamiltonian Monte Carlo}




	\begin{icmlauthorlist}
		\icmlauthor{Nilesh Tripuraneni}{berk}
		\icmlauthor{Mark Rowland}{cam}
		\icmlauthor{Zoubin Ghahramani}{cam,uber}
		\icmlauthor{Richard Turner}{cam}
	\end{icmlauthorlist}

	\icmlaffiliation{cam}{University of Cambridge, UK}
	\icmlaffiliation{berk}{UC Berkeley, USA}
	\icmlaffiliation{uber}{Uber AI Labs, USA}

	\icmlcorrespondingauthor{Nilesh Tripuraneni}{nileshtrip@gmail.com}

	\icmlkeywords{boring formatting information, machine learning, ICML}

	\vskip 0.3in
	]



	\printAffiliationsAndNotice{} 

	\begin{abstract}
		Hamiltonian Monte Carlo (HMC) exploits Hamiltonian dynamics to construct efficient proposals for Markov chain Monte Carlo (MCMC). In this paper, we present a generalization of HMC which exploits \textit{non-canonical} Hamiltonian dynamics.
		We refer to this algorithm as magnetic HMC, since in 3 dimensions a subset of the dynamics map onto the mechanics of a charged particle coupled to a magnetic field.
		We establish a theoretical basis for the use of non-canonical Hamiltonian dynamics in MCMC, and construct a symplectic, leapfrog-like integrator allowing for the implementation of magnetic HMC.
		Finally, we exhibit several examples where these non-canonical dynamics can lead to improved mixing of magnetic HMC relative to ordinary HMC.
	\end{abstract}

	\section{Introduction}
	Probabilistic inference in complex models generally requires the evaluation of intractable, high-dimensional integrals. One powerful and generic approach to inference is to use Markov chain Monte Carlo (MCMC) methods to generate asymptotically exact (but correlated) samples from a posterior distribution for inference and learning.
	Hamiltonian Monte Carlo (HMC) \cite{DuaneEtAl1987, Neal2011} is a state-of-the-art MCMC method which uses gradient information from an absolutely continuous target density to encourage efficient sampling and exploration. Crucially, HMC utilizes proposals inspired by Hamiltonian dynamics (corresponding to the classical mechanics of a point particle)
	which can traverse long distances in parameter space. HMC, and variants like NUTS (which eliminates the need to hand-tune the algorithm's hyperparameters), have been successfully applied to a large class of probabilistic inference problems where they are often the gold standard for (asymptotically) exact inference \cite{Neal1996, Hoffman2014, Carpenter2016}.
	\begin{figure}[!ht]
		\centering
		\begin{minipage}[t]{.23\textwidth}
			\centering
			\includegraphics[width=1\linewidth]{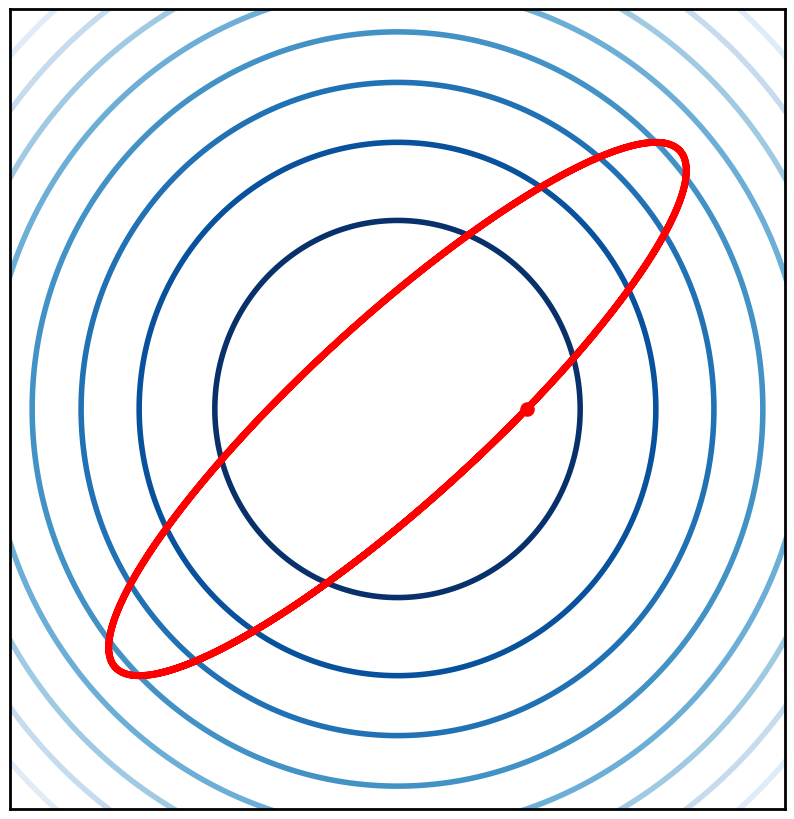}
		\end{minipage}%
		~
		\begin{minipage}[t]{.23\textwidth}
			\centering
			\includegraphics[width=1\linewidth]{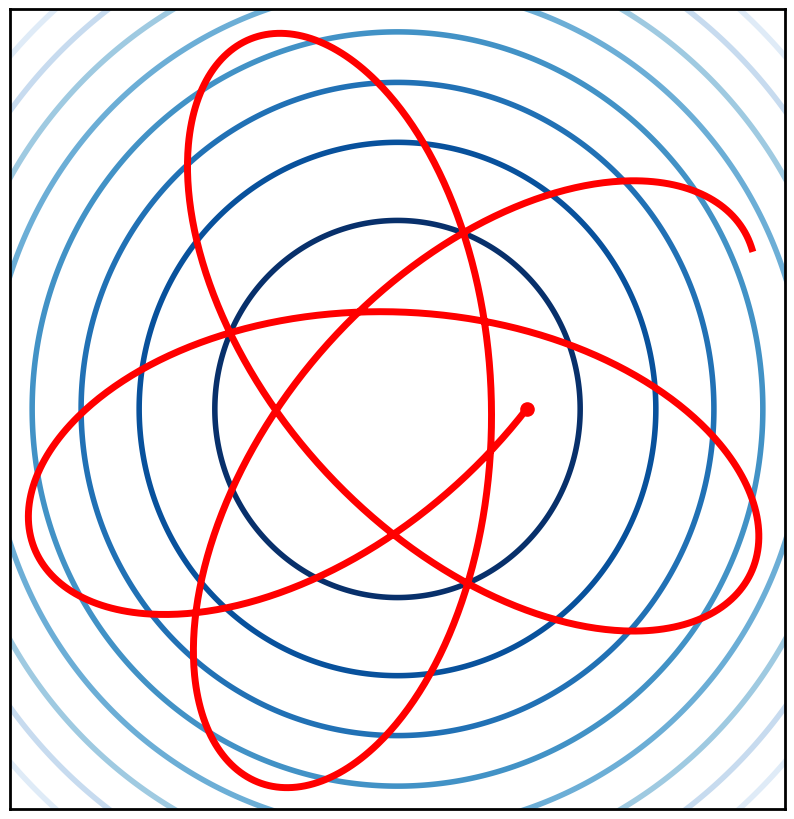}
		\end{minipage}%
		\caption{Example sample paths for standard HMC (left) and MHMC (right) for an isotropic Gaussian target distribution.}
		\label{fig:exampledynamics}
		\vspace{-.5cm}
	\end{figure}

	%

	In this paper, we first review important properties of Hamiltonian dynamics, namely energy-preservation, symplecticity, and time-reversibility, and derive a more general class of dynamics with these properties which we refer to as \textit{non-canonical} Hamiltonian dynamics.
	We then discuss the relationship of non-canonical Hamiltonian dynamics to well-known variants of HMC and propose a novel extension of HMC. We refer to this method as magnetic HMC (see Algorithm \ref{algo:leapfrogNCHMC}) since it corresponds to a particular subset of the non-canonical dynamics
	that in 3 dimensions map onto to the mechanics of a charged particle coupled to a magnetic field -- see Figure \ref{fig:exampledynamics} for an example of these dynamics.
	Furthermore, we construct an explicit, symplectic, leapfrog-like integrator for magnetic HMC which allows for an efficient numerical integration scheme comparable to that of ordinary HMC.
	Finally, we evaluate the performance of magnetic HMC on several sampling problems where we show how its non-canonical dynamics can lead to improved mixing. The proofs of all results in this paper are presented in the corresponding sections of the Appendix.

	\section{Markov chain Monte Carlo}\label{sec:hmc}
	Given an unnormalized target density $\rho(\theta)$ defined on $\mathbb{R}^d$, an MCMC algorithm constructs an ergodic Markov chain $(\Theta_n)_{n \in \mathbb{N}}$ such that the distribution of $\Theta_n$ converges to $\rho$ (e.g.\ in total variation) \cite{Robert2004}.
	Often, the transition kernel of such a Markov chain is specified by the Metropolis-Hastings (MH) algorithm which (i) given the current state $\Theta_n=\theta$, proposes a new state $\widetilde{\theta}$ by sampling from a proposal distribution $Q(\cdot | \theta)$, and (ii) sets $\Theta_{n+1}=\widetilde{\theta}$ with probability $\min\left(1, \frac{\rho(\widetilde{\theta})Q(\theta|\widetilde{\theta})}{\rho(\theta)Q(\widetilde{\theta}|\theta)}\right)$ and $\Theta_{n+1}=\theta$ otherwise. The role of the acceptance step is to enforce reversibility (or detailed balance) of the Markov chain with respect to $\rho$ -- which implies $\rho$ is a stationary distribution of the transition kernel.

	Heuristically, a good MH algorithm should have low inter-sample correlation while maintaining a high acceptance ratio. Hamiltonian Monte Carlo provides an elegant mechanism to do this by simulating a particle moving along the contour lines of a dynamical system, constructed from the target density, to use as a MCMC proposal.
	\subsection{Hamiltonian Monte Carlo}\label{subsec:hmc}
	In Hamiltonian Monte Carlo, the target distribution is augmented with ``momentum'' variables $\bfs{p}$ which are independent of the $\theta$ variables but of equal dimension. For the remainder of the paper, we take the distribution over the momentum variables to be Gaussian, as is common in the literature (indeed, there is evidence that in many cases, the choice of a Gaussian distribution may be optimal \citep{conceptintrohmc}). The joint target distribution is therefore:
	\begin{align}
	\rho(\bfs{\theta}, \bfs{p}) \propto e^{-U(\bfs{\theta})-\bfs{p}^\top\bfs{p}/2} \equiv e^{-H(\theta, \bfs{p})}. \label{eq:joint}
	\end{align}
	Crucially, this augmentation allows Hamiltonian dynamics to be used as a proposal for an MCMC algorithm over the space $(\theta, \bfs{p})$, where we interpret $\theta$ (resp., $\bfs{p}$) as position (resp., momentum) coordinates of a physical particle with total energy $H(\theta, \bfs{p})$, given by the sum of its potential energy $U(\theta)$ and kinetic energy $\bfs{p}^\top \bfs{p}/2$. We briefly review the Markov chain construction below; see \cite{Neal2011} or \cite{DuaneEtAl1987} for a more detailed description. Given the Markov chain state $(\theta_n, \bfs{p}_n)$ at time $n$, the new state for time $n+1$ is obtained by first resampling momentum $\bfs{p}_n \sim \mathcal N(\bfs{0}, \bfs{I})$, and then proposing a new state according to the following steps: (i) Simulate the deterministic
	Hamiltonian flow defined by the differential equation

	\begin{align}
	\rms{\frac{d}{dt}} \begin{bmatrix} \bfs{\theta}(t) \\  \bfs{p}(t) \end{bmatrix} &  = \underbrace{\begin{bmatrix} \bfs{0} & \bfs{I} \\
		\bfs{-I} & \bfs{0} \end{bmatrix}}_{\bfs{A}} \begin{bmatrix} \nabla_{\bfs{\theta}} H(\bfs{p}(t), \bfs{\theta}(t)) \\  \nabla_{\bfs{p}}H(\bfs{p}(t), \bfs{\theta}(t))\end{bmatrix} \nonumber \\
	& \equiv \begin{bmatrix} \bfs{p}(t) \\  -\nabla_{\bfs{\theta}} U(\bfs{\theta}(t)) \end{bmatrix}.
	\label{eq:hd}
	\end{align}

	for time $\tau$, with initial condition $(\theta_n, \bfs{p}_n)$, to obtain $(\theta^\prime_n, \bfs{p^\prime}_n) = \bfs{\Phi}_{\tau, H}(\theta_n, \bfs{p}_n)$\footnote{Throughout this paper, we use $\bfs{\Phi}_{\tau, H}$ to denote the map that takes a given position-momentum pair as initial conditions for the Hamiltonian flow associated with $H$ for time $\tau$. In addition, $\boldsymbol{\widetilde{\Phi}}_{\tau, H}$ denotes the composition of $\boldsymbol{\Phi}_{\tau, H}$ with the momentum flip map $\boldsymbol{\Phi}_\bfs{p}$.}; (ii) Flip the resulting momentum component with the map $\bfs{\Phi}_{\bfs{p}}(\theta, \bfs{p}) = (\theta, -\bfs{p})$ to obtain $(\widetilde{\theta}_{n+1}, \widetilde{\bfs{p}}_{n+1}) = \bfs{\Phi}_{\bfs{p}}(\theta'_n, \bfs{p}'_n) =\bfs{\widetilde{\Phi}}_{\tau, H}(\theta_n, \bfs{p}_n)$; (iii) Apply a MH-type accept/reject step to enforce detailed balance with respect to the target distribution; (iv) Flip the momentum again with $\bfs{\Phi}_{\bfs{p}}$ so it points in the original direction.

	Note that because the map $\bfs{\Phi}_{\tau, H}$ is time-reversible (in the sense that if the path $(\bfs{\theta}(t), \bfs{p}(t))$ is a solution to \eqref{eq:hd} then the path with negated momentum traversed in reverse $(\bfs{\theta}(-t), -\bfs{p}(-t))$ is also a solution), the map $\bfs{\widetilde{\Phi}}_{\tau, H}$ is self-inverse. From this, the acceptance ratio in step (iii) enforcing detailed balance can be shown (see e.g. \cite{Green1995}) to have the form:
	\scriptsize
	\begin{align}
	\min\left(1, \frac{\exp(-H(\widetilde{\theta}_{n+1}, \bfs{\widetilde{p}}_{n+1}))}{\exp(-H(\theta_n, \bfs{p}_n))} \left|\det \nabla_{\theta, \bfs{p}}\bfs{\widetilde{\Phi}}_{\tau, H}(\theta_{n}, \bfs{p}_{n}) \right| \right). \label{eq:accratio}
	\end{align}
	\normalsize
	Note that the Hamiltonian flow \& momentum flip operator $\bfs{\widetilde{\Phi}}_{\tau, H}$ is volume-preserving\footnote{In fact the Hamiltonian flow satisfies the stronger condition of symplecticity with respect to the $\bfs{A}$ matrix ($[\nabla_{\theta, \bfs{p}} \bfs{\Phi}_{\tau, H}(\theta, \bfs{p})]^\top \bfs{A}^{-1} [\nabla_{\theta, \bfs{p}} \bfs{\Phi}_{\tau, H}(\theta, \bfs{p})] = \bfs{A}^{-1}$) which immediately implies it is volume-preserving by taking determinants of this relation.}, which immediately yields that the Jacobian term in the acceptance ratio \eqref{eq:accratio} is simply 1.
	The acceptance probability therefore reduces to $\min(1,\exp( H(\theta_n, \bfs{p}_n) - H(\widetilde{\theta}_{n+1}, \bfs{\widetilde{p}}_{n+1}) ))$. Furthermore, since the Hamiltonian flow defined in \eqref{eq:hd} is energy-preserving (i.e.\ $H(\widetilde{\theta}_{n+1}, \bfs{\widetilde{p}}_{n+1}) = H(\theta_n, \bfs{p}_n)$)  -- the acceptance ratio is identically 1. Moreover, the momentum resampling in (i) and momentum flip in (iv) both leave the joint distribution invariant.

	While the momentum resampling ensures the Markov chain explores the joint $(\theta, \bfs{p})$ space, the proposals inspired by Hamiltonian dynamics can traverse long distances in parameter space $\bfs{\theta}$, reducing the random-walk behavior of MH that often results in highly correlated samples \cite{Neal2011}.
	\subsection{Symplectic Numerical Integration}\label{subsec:hmc-leapfrog}
	Unfortunately, it is rarely possible to integrate the flow defined in \eqref{eq:hd} analytically; instead an efficient numerical integration scheme must be used to generate a proposal for the MH-type accept/reject test. Typically, the leapfrog  (St\"ormer-Verlet) integrator is used since it is an explicit method that is both symplectic and time-reversible \cite{Neal2011}. One elegant way to motivate this integrator is by decomposing the Hamiltonian into a symmetric splitting:
	\begin{align}
	H(\bfs{\theta}, \bfs{p}) = \underbrace{U(\bfs{\theta})/2}_{H_1(\rms{\theta})} + \underbrace{\bfs{p}^\top\bfs{p}/2}_{H_2(\bfs{p})} +  \underbrace{U(\bfs{\theta})/2}_{H_1(\rms{\theta})} \label{eq:split}
	\end{align}
	and then defining $\bfs{\Phi}_{\epsilon, H_1(\rms{\theta})}$ and $\bfs{\Phi}_{\epsilon, H_2(\bfs{p})}$ to be the exactly-integrated flows for the sub-Hamiltonians $H_1(\theta)$ and $H_2(\bfs{p})$, respectively. These updates (which are equivalent to Euler translations) can be written:
	\begin{align}
	& \bfs{\Phi}_{\epsilon, H_1(\rms{\theta})}  \begin{bmatrix} \bfs{\theta} \\ \bfs{p} \end{bmatrix} = \begin{bmatrix} \bfs{\theta} \\ \bfs{p} - \frac{\epsilon}{2} \nabla_{\rms{\theta}}U(\rms{\theta}) \end{bmatrix} \nonumber \\
	& \bfs{\Phi}_{\epsilon, H_2(\bfs{p})}  \begin{bmatrix} \bfs{\theta} \\ \bfs{p} \end{bmatrix} = \begin{bmatrix} \bfs{\theta} + \epsilon \bfs{p} \\ \bfs{p} \end{bmatrix}
	\end{align}
	since the Hamilton equations \eqref{eq:hd} for the sub-Hamiltonians $H_1(\theta)$ and $H_2(\bfs{p})$ are linear, and hence analytically integrable. One leapfrog step is then defined as:
	\begin{align}
	\bfs{\Phi}^{\mathrm{frog}}_{\epsilon, H(\bfs{\theta}, \bfs{p})} = \bfs{\Phi}_{\epsilon, H_1(\rms{\theta})}  \circ \bfs{\Phi}_{\epsilon, H_2(\bfs{p})}  \circ \bfs{\Phi}_{\epsilon, H_1(\rms{\theta})}
	\end{align}
	with the overall proposal given by $L$ leapfrog steps, followed by the momentum flip operator $\bfs{\Phi}_{\bfs{p}}$ as before:
	\[
	\bfs{\widetilde{\Phi}}_{L, \epsilon, H}^{\mathrm{frog}} = \bfs{\Phi}_{\bfs{p}} \circ \left( \bfs{\Phi}_{\epsilon, H(\theta, \bfs{p})}^{\mathrm{frog}} \right)^L .
	\]
	As each of the flows $\bfs{\Phi}_{\epsilon, H_1(\theta)}$, $\bfs{\Phi}_{\epsilon, H_2(\bfs{p})}$ exactly integrates a sub-Hamiltonian, they inherit the symplecticity, volume-preservation, and time-reversibility of the exact dynamics. Moreover, since the composition of symplectic flows is also symplectic and the splitting scheme is symmetric (implying the composition of time-reversible flows is also time-reversible), the Jacobian term in the acceptance probability \eqref{eq:accratio} is exactly 1 as in the case of perfect simulation.

	The leapfrog scheme will not exactly preserve the Hamiltonian $H$, so the remaining acceptance ratio $\exp( H(\theta_n, \bfs{p}_n) - H(\widetilde{\theta}_{n+1}, \widetilde{\bfs{p}}_{n+1}))$ must be calculated. However, the leapfrog integrator has error $\mathcal{O}(\epsilon^3)$ in one leapfrog step \cite{Hairer2006}. This error scaling will lead to good energy conservation properties (and thus high acceptance rates in the MH step), even when simulating over long trajectories.
	\section{Non-Canonical Hamiltonian Monte Carlo}\label{subsec: ncd-properties}
	In Section \ref{sec:hmc}, we noted the role time-reversibility, volume-preservation, and energy conservation of canonical Hamiltonian dynamics play in making them useful candidates for MCMC. In this section, we develop the properties of a general class of flows we refer to as \textit{non-canonical} Hamiltonian systems that parallel these properties,
	we use to construct our method magnetic HMC (see Algorithm \ref{algo:leapfrogNCHMC}):
	\begin{lemma}\label{lem:nchd-1}
		The map $\bfs{\Phi}_{\tau, H}^{\bfs{A}}(\theta, \bfs{p})$ defined by integrating the non-canonical Hamiltonian system
		\begin{align}
		\rms{\frac{d}{dt}} \begin{bmatrix} \bfs{\theta}(t) \\  \bfs{p}(t) \end{bmatrix}  = \bfs{A} \nabla_{\theta, \bfs{p}} H(\theta(t), \bfs{p}(t)) \label{eq:ncd}
		\end{align}
		with initial conditions $(\theta, \bfs{p})$ for time $\tau$, where $\bfs{A} \in \mathcal{M}_{2n\times2n}$ is \textit{any} invertible, antisymmetric matrix induces a flow on the coordinates $(\bfs{\theta}, \bfs{p})$ that is still \textit{energy-conserving} $(\partial_{\tau} H(\bfs{\Phi}_{\tau, H}^{\bfs{A}}(\theta, \bfs{p})) = 0)$ and \textit{symplectic} with respect to $\bfs{A}$ $([\nabla_{\theta, \bfs{p}} \bfs{\Phi}_{\tau, H}(\theta, \bfs{p})]^\top \bfs{A}^{-1} [\nabla_{\theta, \bfs{p}} \bfs{\Phi}_{\tau, H}(\theta, \bfs{p})] = \bfs{A}^{-1})$ which also implies volume-preservation of the flow.
	\end{lemma}
	Within the formal construction of classical mechanics, it is known that any Hamiltonian flow defined on the cotangent bundle $(\theta, \bfs{p})$ of a configuration manifold, which is equipped with an arbitrary symplectic 2-form, will preserve its symplectic structure and admit the corresponding Hamiltonian as a first integral invariant \cite{Arnold1989}. The statement of Lemma \ref{lem:nchd-1} is simply a restatement of this fact grounded in a coordinate system. Similar arbitrary, antisymmetric terms have also appeared in the study of MCMC algorithms based on diffusion processes; such samplers often do not enforce detailed balance with respect to the target density and are often implemented as discretizations of stochastic differential equations \cite{Rey-Bellet2015, Ma2015}, in contrast to the approach taken here.

	Our second observation is that the dynamics in \eqref{eq:ncd} are \textit{not} time-reversible in the traditional sense.
	Instead, if we consider the parametrization of $\bfs{A}$ as:
	\begin{align}
	\bfs{A} = \begin{bmatrix} \topleftmatrix & \offdiagmatrix \\
	-\offdiagmatrix^\top & \bottomrightmatrix \end{bmatrix}
	\end{align}
	where $\topleftmatrix$, $\bottomrightmatrix$ are antisymmetric and $\offdiagmatrix$ is taken to be general such that $\bfs{A}$ is invertible, then the \textit{non-canonical} dynamics have a (pseudo) time-reversibility symmetry:
	\begin{lemma}\label{lem:nchd-2}
		If  $(\theta(t), \bfs{p}(t))$ is a solution to the non-canonical dynamics:
		\begin{align}
		\rms{\frac{d}{dt}} \begin{bmatrix} \bfs{\theta}(t) \\  \bfs{p}(t) \end{bmatrix}  = \underbrace{\begin{bmatrix} \topleftmatrix & \offdiagmatrix \\
			-\offdiagmatrix^\top & \bottomrightmatrix \end{bmatrix}}_{\bfs{A}} \begin{bmatrix} \nabla_{\bfs{\theta}} H(\theta(t), \bfs{p}(t)) \\  \nabla_{\bfs{p}} H(\theta(t), \bfs{p}(t))  \end{bmatrix} \label{eq:timereverse}
		\end{align}
		then $(\widetilde{\theta}(t), \widetilde{\bfs{p}}(t)) = (\theta(-t), -\bfs{p}(-t))$ is a solution to the modified non-canonical dynamics:
		\begin{align}
		\rms{\frac{d}{dt}} \begin{bmatrix} \widetilde{\bfs{\theta}}(t) \\  \widetilde{\bfs{p}}(t) \end{bmatrix}  = \underbrace{\begin{bmatrix} -\topleftmatrix & \offdiagmatrix \\
			-\offdiagmatrix^\top & -\bottomrightmatrix \end{bmatrix}}_{\bfs{\widetilde{A}}} \begin{bmatrix} \nabla_{\widetilde{\bfs{\theta}}} H(\widetilde{\theta}(t), \bfs{p}(t)) \\  \nabla_{\widetilde{\bfs{p}}} H(\widetilde{\theta}(t), \widetilde{\bfs{p}}(t)) \end{bmatrix} \label{timereverse2}
		\end{align}
		if $H(\bfs{\theta}, \bfs{p}) = H(\bfs{\theta}, -\bfs{p})$. In particular if $\topleftmatrix=\bottomrightmatrix=0$ then $\bfs{A}=\bfs{\widetilde{A}}$, which reduces to the traditional time-reversal symmetry of canonical Hamiltonian dynamics.
	\end{lemma}

	Lemma \ref{lem:nchd-1} suggests a generalization of HMC that can utilize an arbitrary invertible antisymmetric $\bfs{A}$ matrix in its dynamics; however Lemma \ref{lem:nchd-2} indicates the non-canonical dynamics lack a traditional time-reversibility symmetry which poses a potential difficulty to satisfying detailed balance. In particular, we cannot compose $\bfs{\Phi}_{\bfs{p}}$
	with an exact/approximate simulation of $\bfs{\Phi}_{\tau, H}^{\bfs{A}}$ to make $\bfs{\widetilde{\Phi}}_{\tau, H}^{\bfs{A}} = \bfs{\Phi}_{\bfs{p}} \circ \bfs{\Phi}_{\tau, H}^{\bfs{A}}$ self-inverse.

	Our solution to obtaining a time-reversible proposal is simply to flip the elements of the $\topleftmatrix$ and $\bottomrightmatrix$ matrices just as ordinary HMC flips the auxiliary variable $\bfs{p}$ i) at the end of Hamiltonian flow in the proposal and ii) once again after the MH acceptance step to return $\bfs{p}$ to its original direction. In this vein, we view the parameters $\bfs{\topleftmatrix}$ and $\bfs{\bottomrightmatrix}$ as auxiliary variables in the state space, and simultaneously flip $\bfs{p}$, $\topleftmatrix$, and $\bottomrightmatrix$ after having simulated the dynamics, rendering the proposal time-reversible according to Lemma \ref{lem:nchd-2} --
	see Section 2 in the Appendix for full details of this construction.
	This ensures that detailed balance is satisfied for this entire proposal. To avoid ``random walk'' behaviour in the resulting Markov chain,  we can apply a sign flip to $\topleftmatrix$ and $\bottomrightmatrix$, in addition to $\bfs{p}$, to return them to their original directions after the MH acceptance step.

	The validity of this construction relies on equipping $\topleftmatrix$ and $\bottomrightmatrix$ with symmetric auxiliary distributions. For the remainder of this paper,
	we further restrict to binary symmetric auxiliary distributions supported on a given antisymmetric matrix $\bfs{V}_0$ and its sign flip $-\bfs{V}_0$ -- see Appendix 1.1 
	for full details. This restriction is not necessary, but gives rise to a simple and interpretable class of algorithms, which is in spirit closest to using fixed parameters $\topleftmatrix$ and $\bottomrightmatrix$, whilst ensuring the proposal satisfies detailed balance. This construction is also reminiscent of lifting constructions prevalent in the discrete Markov chain literature \cite{Lifting}; heuristically, the signed variables $\topleftmatrix$ and $\bottomrightmatrix$ favour proposals in opposing directions.

	\subsection{Symplectic Numerical Integration for Non-Canonical Dynamics} \label{subsec: ncd-leapfrog}
	As with standard HMC, exactly simulating the flow $\bfs{\Phi}_{\tau, H}^{\bfs{A}}$ is rarely tractable, and a numerical integrator is required to approximate the flow. It is not absolutely necessary to use an explicit, symplectic integration scheme; indeed implicit integrators are used in Riemannian HMC to maintain symplecticity of the proposal which comes at a greater complexity and computational cost \cite{Girolami2009}. However explicit, symplectic integrators are simple, have good energy-conservation properties, and are volume-preserving/time-reversible \cite{Hairer2006}, so for the present discussion we restrict our attention to investigating leapfrog-like schemes.

	We begin, as in Section \ref{subsec:hmc-leapfrog}, by considering the symmetric splitting \eqref{eq:split}, yielding the sub-Hamiltonians $H_1(\theta) = U(\theta)/2$, $H_2(\bfs{p}) = \bfs{p}^\top \bfs{p}/2$. The corresponding non-canonical dynamics for the sub-Hamiltonians $H_1(\theta)$ and $H_2(\bfs{p})$ are:
	\vspace{-.25cm}
	\begin{align}
	& \rms{\frac{d}{dt}} \begin{bmatrix} \bfs{\theta} \\  \bfs{p} \end{bmatrix}  = \underbrace{\begin{bmatrix} \topleftmatrix & \offdiagmatrix \\
		-\offdiagmatrix^\top & \bottomrightmatrix \end{bmatrix}}_{\bfs{A}} \begin{bmatrix} \nabla_{\bfs{\theta}}U(\bfs{\theta})/2 \\  \bfs{0} \end{bmatrix} = \begin{bmatrix} \topleftmatrix \nabla_{\bfs{\theta}}U(\bfs{\theta})/2  \\  -\offdiagmatrix^\top \nabla_{\bfs{\theta}}U(\bfs{\theta})/2 \end{bmatrix} \nonumber
	\end{align}
	and:
	\begin{align}
	& \rms{\frac{d}{dt}} \begin{bmatrix} \bfs{\theta} \\  \bfs{p} \end{bmatrix}  = \underbrace{\begin{bmatrix} \topleftmatrix & \offdiagmatrix \\
		-\offdiagmatrix^\top & \bottomrightmatrix \end{bmatrix}}_{\bfs{A}} \begin{bmatrix} \bfs{0} \\  \bfs{p} \end{bmatrix} = \begin{bmatrix} \offdiagmatrix \bfs{p}  \\  \bottomrightmatrix \bfs{p} \end{bmatrix}. \nonumber
	\end{align}
	We denote the corresponding flows by $\bfs{\Phi}_{\epsilon, H_1(\theta)}^\bfs{A}$ and $\bfs{\Phi}_{\epsilon, H_2(\bfs{p})}^\bfs{A}$ respectively.
	The flow $\bfs{\Phi}_{\epsilon, H_1(\theta)}^\bfs{A}$ is generally not explicitly tractable unless we take $\topleftmatrix= \bfs{0}$ -- in which case it is solved by an Euler translation as before. Crucially, the flow in $\bfs{\Phi}_{\epsilon, H_2(\bfs{p})}^\bfs{A}$ is a \textit{linear} differential equation and hence analytically integrable.
	If $\bottomrightmatrix$ is invertible (and $\offdiagmatrix = \bfs{I}$)  then:
	\begin{align}
	\bfs{\Phi}_{\epsilon, H_2(\bfs{p})}  \begin{bmatrix} \bfs{\theta} \\ \bfs{p} \end{bmatrix} = \begin{bmatrix} \bfs{\theta} + \bottomrightmatrix^{-1}(\exp(\bottomrightmatrix\epsilon) - \bfs{I})\bfs{p} \\ \exp(\bottomrightmatrix\epsilon)\bfs{p} \end{bmatrix}.
	\label{eq:matrixexp}
	\end{align}
	See the Appendix for a detailed derivation which also handles the general case where $\bottomrightmatrix$ is not invertible. Thus when $\topleftmatrix=\bfs{0}$, the flows $\bfs{\Phi}^\bfs{A}_{\epsilon, H_1(\theta)}$ and $\bfs{\Phi}^\bfs{A}_{\epsilon, H_2(\bfs{p})}$ are analytically tractable and will inherit the generalized symplecticity and (pseudo) time-reversibility of the exact dynamics in \eqref{eq:ncd}. Therefore if we use the symmetric splitting \eqref{eq:split} to construct a leapfrog-like step:
	\begin{align}\label{eq:nchmc_leapfrog1}
	\bfs{\Phi}_{\epsilon, H(\theta, \bfs{p})}^\bfs{\mathrm{frog}, A} = \bfs{\Phi}^\bfs{A}_{\epsilon, H_1(\theta)} \circ \bfs{\Phi}^\bfs{A}_{\epsilon, H_2(\bfs{p})} \circ \bfs{\Phi}^\bfs{A}_{\epsilon, H_1(\theta)}
	\end{align}
	we can construct a total proposal that consists of several leapfrog steps, followed by a flip of the momentum and $\bottomrightmatrix$, $\bfs{\Phi}_{\bfs{p}} \circ \bfs{\Phi}_{\bfs{\bottomrightmatrix}}$, which will be a volume-preserving, self-inverse map:
	\scriptsize
	\begin{align}\label{eq:nchmc_leapfrog2}
	\bfs{\widetilde{\Phi}}_{\epsilon, H(\theta, \bfs{p})}^\bfs{\mathrm{frog}, A} = \bfs{\Phi}_{\bfs{p}} \circ \bfs{\Phi}_{\bottomrightmatrix} \circ (\bfs{\Phi}^\bfs{A}_{\epsilon, H_1(\theta)} \circ \bfs{\Phi}^\bfs{A}_{\epsilon, H_2(\bfs{p})} \circ \bfs{\Phi}^\bfs{A}_{\epsilon, H_1(\theta)})^{L}.
	\end{align}
	\normalsize
	Henceforth we will always take $\bfs{E}=\bfs{0}$ when we use $\bfs{\Phi}_{\epsilon, H(\theta, \bfs{p})}^\bfs{\mathrm{frog}, A}$ to generate leapfrog proposals, which interestingly corresponds to a magnetic dynamics as discussed in Lemma \ref{lem:magnetic}. A full description of the magnetic HMC algorithm using this numerical integrator is described in Section \ref{sec: algo}.
	\section{Special Cases}
	Here, we describe several tractable subcases of the general formulation of non-canonical Hamiltonian dynamics since these they have interesting physical interpretations.
	\subsection{Mass Preconditioned Dynamics} \label{subsec: mass-hmc}
	One simple variant of HMC is preconditioned HMC where $\bfs{p} \sim \mathcal{N}(0, \bfs{M})$ \cite{Neal2011}, and can be implemented nearly identically to ordinary HMC.
	We note that preconditioning can be recovered within our framework using a simple form for the non-canonical $\bfs{A}$ matrix:
	\begin{lemma}\label{lem:preconditioned}
		i) Preconditioned HMC with momentum variable $\bfs{p} \sim \mathcal{N}(0, \bfs{M})$ in the $(\theta, \bfs{p})$ coordinates is exactly equivalent to simulating non-canonical HMC with $\bfs{p}' = \bfs{M}^{-1/2} \bfs{p} \sim \mathcal{N}(\bfs{0}, \bfs{I})$ and the non-canonical matrix $
		\bfs{A} = \begin{bmatrix} \bfs{0} & \bfs{M}^{1/2}  \\ -(\bfs{M}^{1/2})^\top & \bfs{0} \end{bmatrix}
		$
		, and then transforming back to $(\theta, \bfs{p})$ coordinates using $\bfs{p} = \bfs{M}^{1/2} \bfs{p}'$. Here $\bfs{M}^{1/2}$ is a Cholesky factor for $\bfs{M}$.

		ii) On the other hand if we apply a change of basis (via an invertible matrix $\bfs{F}$) to our coordinates $\theta' = \bfs{F}^{-1} \theta$, simulate HMC in the $(\theta', \bfs{p})$ coordinates, and transform back to the original basis using $\bfs{F}$, this is exactly equivalent to non-canonical HMC with
		$
		\bfs{A} = \begin{bmatrix} \bfs{0} & \bfs{F} \\ -\bfs{F}^\top & \bfs{0}  \end{bmatrix}.
		$
	\end{lemma}
	Lemma \ref{lem:preconditioned} illustrates a fact alluded to in \cite{Neal2011}; using a mass preconditioning matrix $\bfs{M}$ and a change of basis given by matrix $-(\bfs{M}^{-1/2})^\top$ acting on $\bfs{\theta}$ leaves the HMC algorithm invariant.
	\subsection{Magnetic Dynamics} \label{subsec: magnetic-hmc}
	The primary focus of this paper is to investigate the subcase of the dynamics where:
	\begin{align}
	\bfs{A} = \begin{pmatrix} \bfs{0} & \bfs{I} \\ -\bfs{I} & \bottomrightmatrix \end{pmatrix} \label{eq:magneticmatrix}
	\end{align}
	for two important reasons\footnote{Note that the effect of a non-identity $\offdiagmatrix$ matrix can be achieved by simply composing these magnetic dynamics with a coordinate-transformation as suggested in Lemma \ref{lem:preconditioned}.}. Firstly for this choice of $\bfs{A}$ we can construct an explicit, symplectic, leapfrog-like integration scheme which is important for developing an efficient HMC sampler as discussed in Section \ref{subsec: ncd-leapfrog}. Secondly, the dynamics have an interesting physical interpretation quite distinct from mass preconditioning and other HMC variants like Riemannian HMC \cite{Girolami2009}:
	\begin{lemma}\label{lem:magnetic}
		In 3 dimensions the non-canonical Hamiltonian dynamics corresponding to the Hamiltonian $H(\bfs{\theta}, \bfs{p}) =  U(\bfs{\theta}) + \frac{1}{2} \bfs{p}^\top \bfs{p}$ and $\bfs{A}$ matrix as in \eqref{eq:magneticmatrix} are
		%
		equivalent to the Newtonian mechanics of a charged particle (with unit mass and charge) coupled to a magnetic field $\vec{\bfs{B}}$ (given by a particular function of $\bottomrightmatrix$ - see Appendix): $
		\frac{\rms{d}^2 \theta}{\rms{dt^2}} = -\nabla_{\bfs{\theta}}U(\bfs{\theta})+ \frac{\rms{d} \theta}{\rms{dt}} \times \vec{\bfs{B}}.
		$
	\end{lemma}
	This interpretation is perhaps surprising since Hamiltonian formulations of classical magnetism are uncommon, although the quantum mechanical treatment naturally incorporates a Hamiltonian framework. However, in light of Lemma \ref{lem:preconditioned} we might wonder if by a clever rewriting of the Hamiltonian we can reproduce this system of ODEs using the canonical $\bfs{A}$ matrix (i.e. $\bfs{E} = \bfs{G} = \bfs{0}$, $\bfs{F} = \bfs{I}$). This is not the case:
	\begin{lemma}\label{lem:impossible}
		The non-canonical Hamiltonian dynamics with magnetic $\bfs{A}$ and Hamiltonian $H(\bfs{\theta}, \bfs{p}) =  U(\bfs{\theta}) + \frac{1}{2} \bfs{p}^\top \bfs{p}$ cannot be obtained using canonical Hamiltonian dynamics for any choice of smooth Hamiltonian. (See Appendix).
	\end{lemma}
	\section{The Magnetic HMC (MHMC) Algorithm} \label{sec: algo}
	Using the results discussed in Section \ref{subsec: ncd-properties} and Section \ref{subsec: ncd-leapfrog} we can now propose Magnetic HMC -- see Algorithm \ref{algo:leapfrogNCHMC}.
	\begin{algorithm}
		\begin{algorithmic}[ht!]
			\REQUIRE{$H$, $\bottomrightmatrix$, $L$, $\epsilon$}
			\STATE{Initialize $(\theta_0, \bfs{p}_0)$, and set $\bottomrightmatrix_0 \gets \bottomrightmatrix$}
			\FOR{$n=1,\ldots,N$}
			\STATE{Resample $\bfs{p}_{n-1} \sim N(\bfs{0}, \bfs{I})$}
			\STATE{Set $(\widetilde{\theta}_n, \widetilde{\bfs{p}}_n) \gets \textsc{LF}(H, L, \epsilon, (\theta_{n-1}, \bfs{p}_{n-1}, \bottomrightmatrix_{n-1}))$ with $\bfs{\Phi}^\bfs{A}_{\epsilon, H_{2}(\bfs{p})}$ as in \eqref{eq:matrixexp}}
			\STATE{Flip momentum $(\widetilde{\theta}_n, \widetilde{\bfs{p}}_n) \gets (\widetilde{\theta}_n, -\widetilde{\bfs{p}}_n) $ and set $\widetilde{\bottomrightmatrix}_n \gets -\bottomrightmatrix_{n-1}$}
			\IF{$\mathrm{Unif}(\lbrack 0,1 \rbrack) < \min(1, \exp(H(\theta_{n-1}, \bfs{p}_{n-1})-H(\widetilde{\theta}_{n}, \bfs{\widetilde{p}}_{n})))$}
			\STATE{
				Set $(\theta_n, \bfs{p}_n, \bottomrightmatrix_n) \gets (\widetilde{\theta}_n, \bfs{\widetilde{p}}_n, \widetilde{\bottomrightmatrix}_{n})$
			}
			\ELSE
			\STATE{
				Set $(\theta_n, \bfs{p}_n, \bottomrightmatrix_n) \gets (\theta_{n-1}, \bfs{p}_{n-1}, \bottomrightmatrix_{n-1})$
			}
			\ENDIF
			\STATE{Flip momentum $\bfs{p}_n \gets -\bfs{p}_n$ and flip $\bottomrightmatrix_n \gets -\bottomrightmatrix_n$}
			\ENDFOR
			\ENSURE $(\theta_n)_{n=0}^N$
			\caption{Magnetic HMC (MHMC)}
			\label{algo:leapfrogNCHMC}
		\end{algorithmic}
	\end{algorithm}

	One further remark is that by construction the integrator for magnetic HMC is expected to have similarly good energy conservation properties to the integrator of standard HMC:
	\begin{lemma}\label{lem:error}
		The symplectic leapfrog-like integrator for magnetic HMC will have the same local ($\sim \mathcal{O}(\epsilon^3)$) and global ($\sim \mathcal{O}(\epsilon^2)$) error scaling (over $\tau \sim \frac{L}{\epsilon}$ steps), as the canonical leapfrog integrator of standard HMC if the Hamiltonian is separable. (See Appendix).
	\end{lemma}
	\vspace{-.3cm}
	It is worthwhile to contrast the algorithmic differences between magnetic HMC and ordinary HMC. Intuitively, the role of the flow $\bfs{\Phi}^\bfs{A}_{\epsilon, H_{2}(\bfs{p})}$ -- which reduces to the standard Euler translation update of ordinary HMC when $\bottomrightmatrix=\bfs{0}$ -- is to introduce a \textit{rotation} into the momentum space of the flow. In particular, a non-zero element $\bfs{G}_{ij}$ will allow momentum to periodically flow between $p_{i}$ and $p_{j}$. If we regard $\bottomrightmatrix$ as an element in the Lie algebra of antisymmetric matrices, which can be thought of as infinitesimal rotations, then the exponential map $\exp(\bottomrightmatrix \epsilon)$ will project this transformation into the Lie group of real orthogonal linear maps.

	With respect to computational cost, although magnetic HMC requires matrix exponentiation/diagonalization to simulate $\bfs{\Phi}^\bfs{A}_{\epsilon, H_{2}(\bfs{p})}$, this only needs to be computed once upfront for $\pm \bottomrightmatrix$ and cached; moreover, as $\pm \bottomrightmatrix$ is diagonalizable, the exact exponential can be calculated in $\mathcal{O}(d^3)$ time. Additionally, there is an $\mathcal{O}(d^2)$ cost for the matrix-vector products needed to implement the flow $\bfs{\Phi}^\bfs{A}_{\epsilon, H_{2}(\bfs{p})}$ as with preconditioning. However, it is possible to design sparsified matrix representations of $\bfs{A}$ which will translate into sparsified rotations if we only wish to "curl" in a specific subspace of dimension $d_0$ -- which will translate into a computational cost of $\mathcal{O}(d_0^3)$ and $\mathcal{O}(d_0^2)$ respectively.

	An important problem to address is the selection of the $\bottomrightmatrix$ matrix, which affords a great deal of flexibility to MHMC relative to HMC; this point is further discussed in the Experiments section, where we argue that in certain cases intuitive heuristics can be used to select the $\bottomrightmatrix$ matrix.
	%
	\section{Experiments}\label{sec:experiments}
	Here we investigate the performance of magnetic HMC against standard HMC in several examples; in each case commenting on our choice of the magnetic field term $\bottomrightmatrix$. Step sizes ($\epsilon$) and number of leapfrog steps ($L$) were tuned to achieve an acceptance rate between $.7-.8$, after which the norm of the non-zero elements in $\bfs{G}$ was set to $\sim.1-.2$ which was found to work well.

	In the Appendix we also display illustrations of different MHMC proposals across several targets in order to provide more intuition for MHMC's dynamics.
	Further experimental details and an additional experiment on a Gaussian funnel target are also provided in the Appendix.
	\subsection{Multiscale Gaussians}
	We consider two highly ill-conditioned Gaussians similar to as in \cite{Sohl-Dickstein2014} to illustrate a heuristic for $\bottomrightmatrix$ matrix selection and demonstrate properties of the magnetic dynamics. In particular we consider a centered, uncorrelated 2D Gaussian with covariance eigenvalues of $10^6$ and $1$ as well as a centered, uncorrelated 10D Gaussian with two large covariance eigenvalues of $10^6$ and remaining eigenvalues of $1$.
	We denote their coordinates as $\bfs{x} = (x_1, x_2) \in \mathbb{R}^2$ and $\bfs{x} = (x_1, \hdots, x_{10}) \in \mathbb{R}^{10}$ respectively.
	\begin{figure}[!htb]
		\centering
		\begin{minipage}[t]{.25\textwidth}
			\centering
			\includegraphics[keepaspectratio, width=1\textwidth]{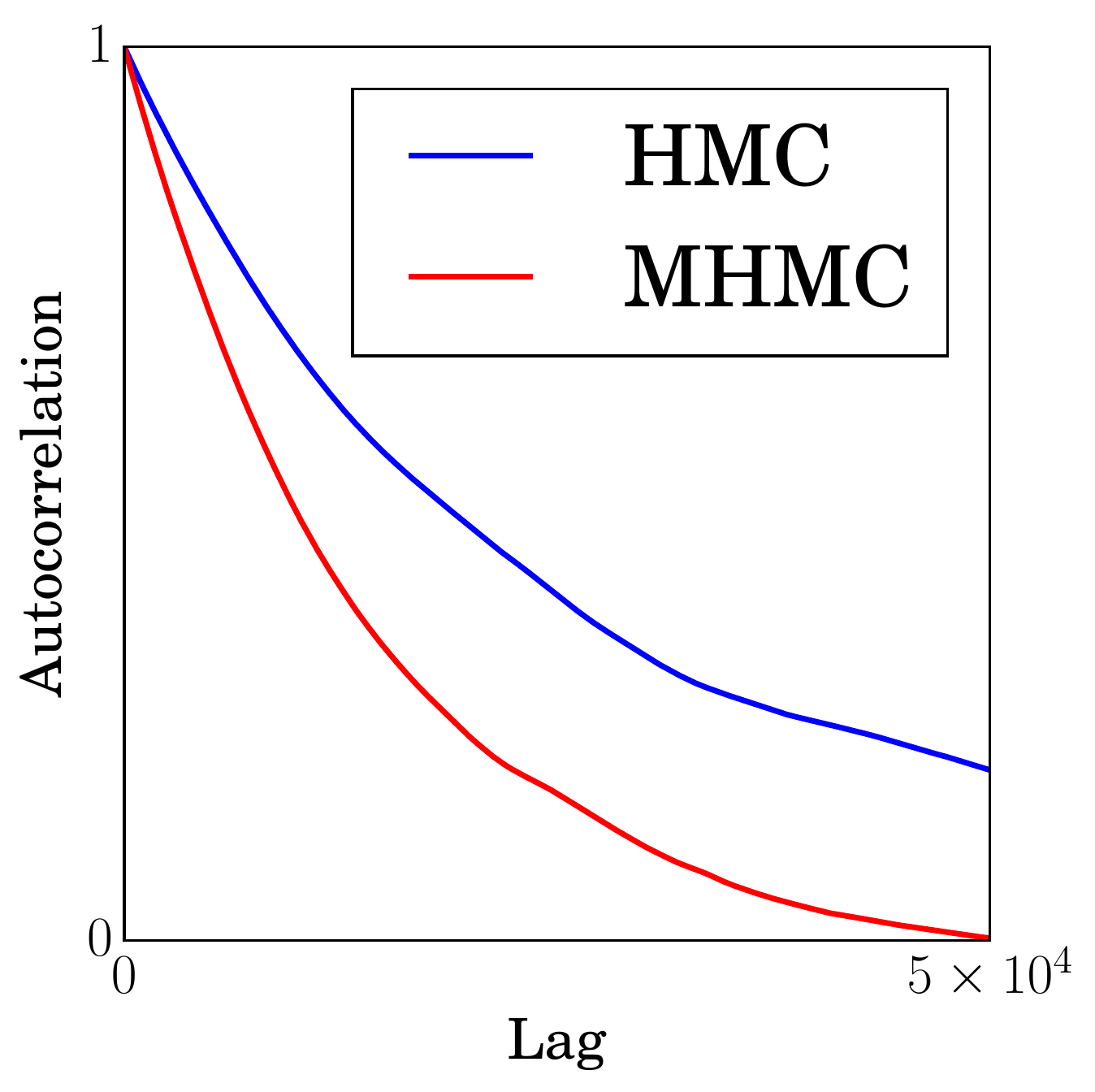}
			\label{samples}
		\end{minipage}%
		~
		\begin{minipage}[t]{.25\textwidth}
			\centering
			\includegraphics[keepaspectratio, width=1\textwidth]{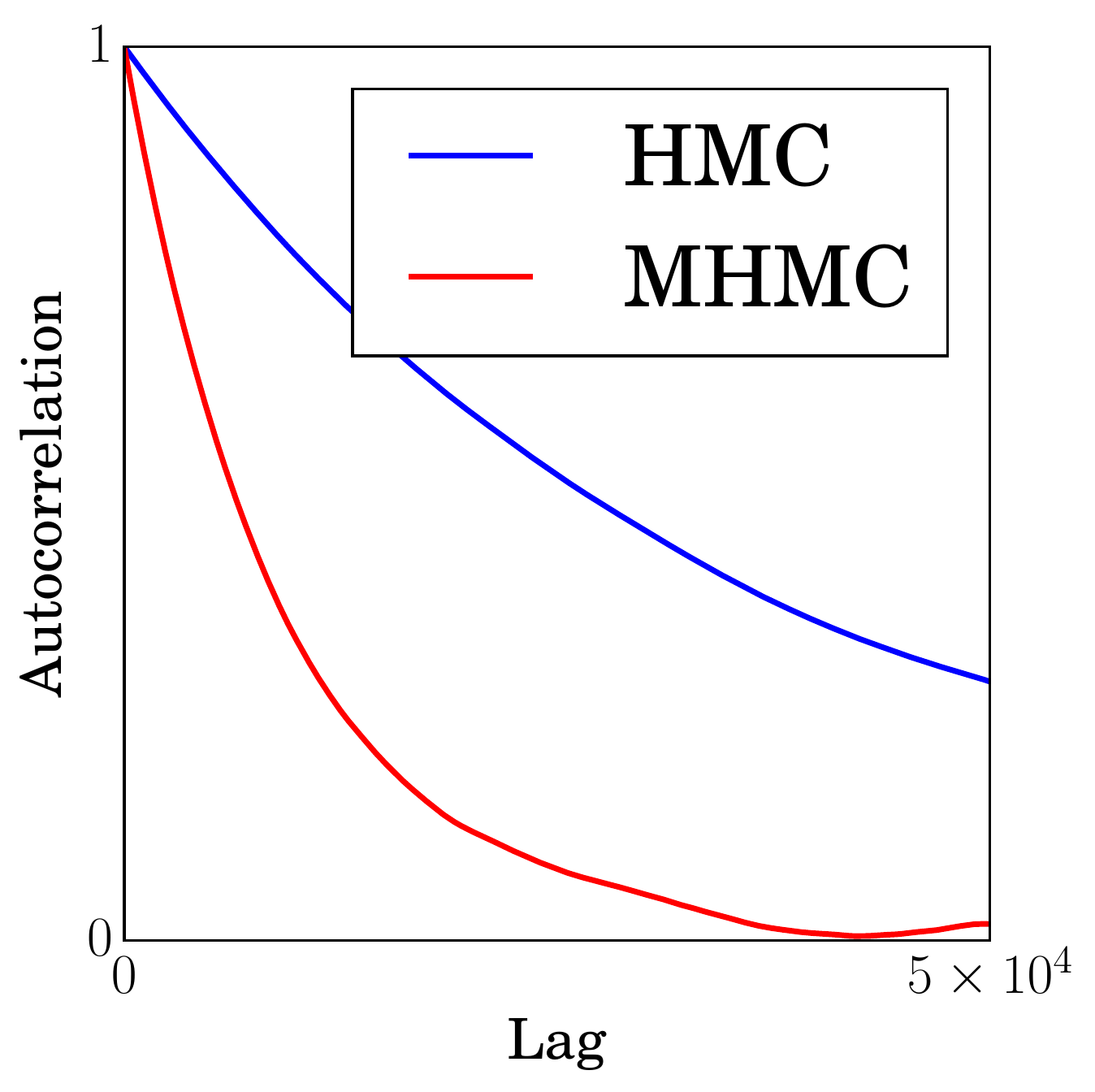}
			\label{samples}
		\end{minipage}%
		\caption{Averaged Autocorrelation of HMC vs MHMC on a 10D ill-conditioned Gaussian (left) and Averaged Autocorrelation of HMC vs MHMC on a 2D ill-conditioned Gaussian.}
		\label{fig:multiscaleGaussians}
	\end{figure}
	%
	HMC will have difficulty exploring the directions of large marginal variance since its exploration will often be limited by the smaller variance directions. Accordingly, in order to induce a periodic momentum flow between the directions of small and large variance, we introduce nonzero components $\bfs{G}_{ij}$ into the subspaces spanned directly between the large and small eigenvalues. Indeed, we find that magnetic $\bfs{G}$ term is encouraging more efficient exploration as we can see from the averaged autocorrelation of samples generated from the HMC/MHMC chains -- see Figure \ref{fig:multiscaleGaussians}. Further, by running the 50 parallel chains for $10^7$ timesteps, we computed both the bias and Monte Carlo standard errors (MCSE) of the estimators of the target moments as shown in Table \ref{fig:1Gaussian2d} and Table \ref{fig:1Gaussian10d}.
	\begin{table}[!ht]
		\vspace{-.25cm}
		\centering
		\caption{Absolute Bias $\pm$ MCSE for 2D ill-conditioned Gaussian moments for HMC vs. MHMC. Note that $\mathbb{E}[x_1^2]=10^6$ and $\mathbb{E}[x_2^2]=1$.}
		\begin{tabular}{cccc} \\\toprule
			algorithm  & $x_{1}^2$ (Bias $\pm$ MCSE) & $x_{2}^2$ (Bias $\pm$ MCSE)  \\ \midrule
			HMC & .947 $\pm$ 41.7 & .00108 $\pm$ .0247 \\  \midrule
			MHMC & .657 $\pm$ 22.5 & .000280 $\pm$ .00438 \\ \midrule
		\end{tabular}
		\label{fig:1Gaussian2d}
		\vspace{-.25cm}
	\end{table}
	\begin{table}[!ht]
		\vspace{-.25cm}
		\centering
		\caption{Absolute Bias $\pm$ MCSE for 10D ill-conditioned Gaussian moments for HMC vs. MHMC. Note that $\mathbb{E}[x_1^2]=10^6$ and $\mathbb{E}[x_{10}^2]=1$.}
		\begin{tabular}{cccc} \\\toprule
			algorithm  & $x_{1}^2$ (Bias $\pm$ MCSE) & $x_{10}^2$ (Bias $\pm$ MCSE)  \\ \midrule
			HMC & 24.7 $\pm$ 46.6 & 0.00376 $\pm$ 0.0249 \\  \midrule
			MHMC & 9.68 $\pm$ 32.7 & 0.00127 $\pm$ 0.0132 \\ \midrule
		\end{tabular}
		\label{fig:1Gaussian10d}
		\vspace{-.25cm}
	\end{table}

	\subsection{Mixture of Gaussians}
	We compare MHMC vs. HMC on a simple, but interesting, 2D density over $\bfs{x} = (x, y) \in \mathbb{R}^2$ comprised of an evenly weighted mixture of isotropic Gaussians: $p(\bfs{x}) = \frac{1}{2}\mathcal{N}(\bfs{x};\boldsymbol{\mu}, \Sigma) + \frac{1}{2}\mathcal{N}(\bfs{x};-\boldsymbol{\mu}, \Sigma)$ for  $\sigma_{x}^2=\sigma_{y}^2=1$, $\rho_{x y}=0$ and $\boldsymbol{\mu}= (2.5, -2.5)$. This problem is challenging for HMC because the gradients in canonical Hamiltonian dynamics force it to one of the two modes.
	\begin{figure*}[!ht]
		\centering
		\begin{minipage}[t]{.32\textwidth}
			\centering
			\includegraphics[width=1\linewidth]{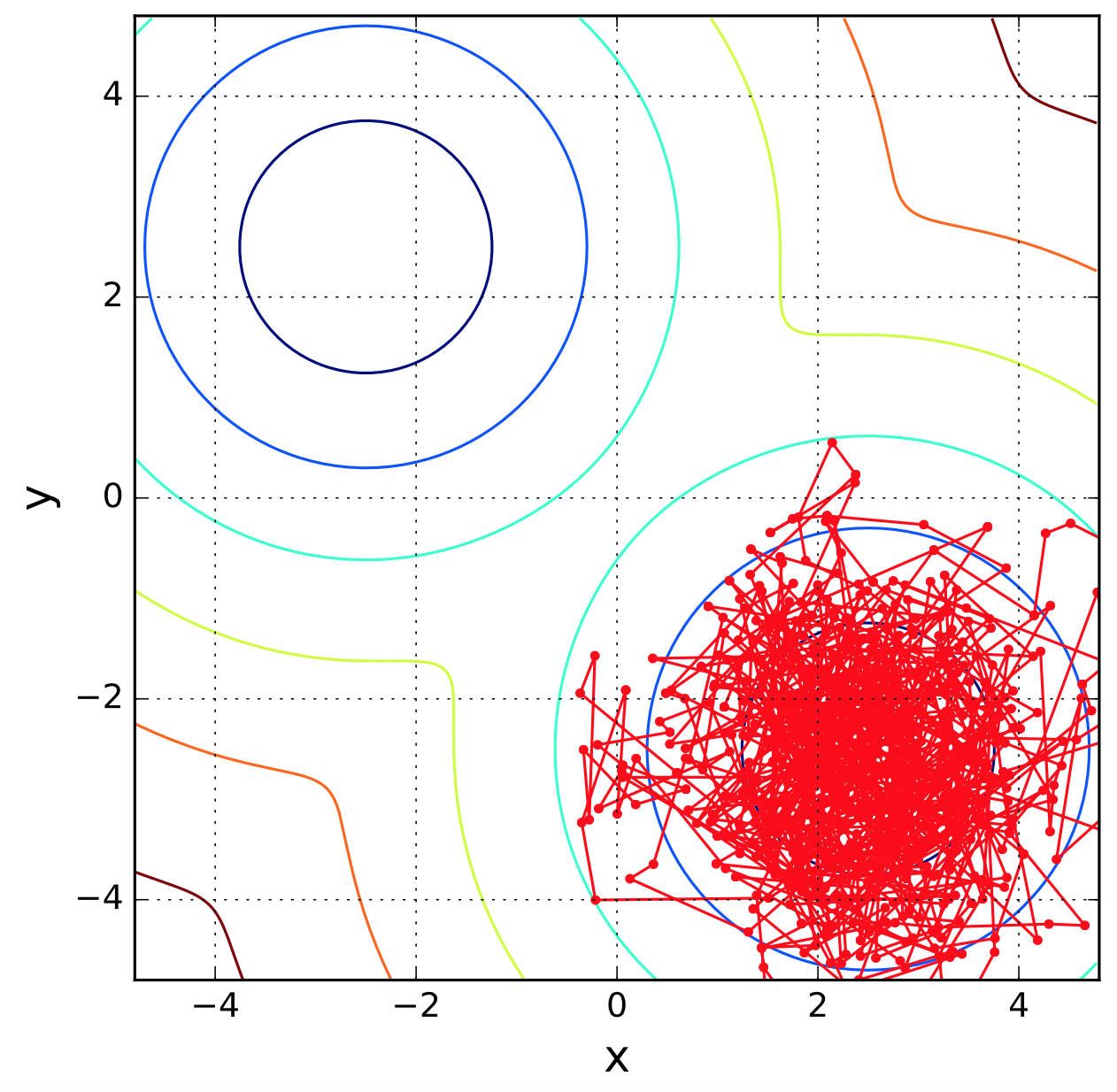}
			\label{samples}
		\end{minipage}%
		~
		\begin{minipage}[t]{.32\textwidth}
			\centering
			\includegraphics[width=1\linewidth]{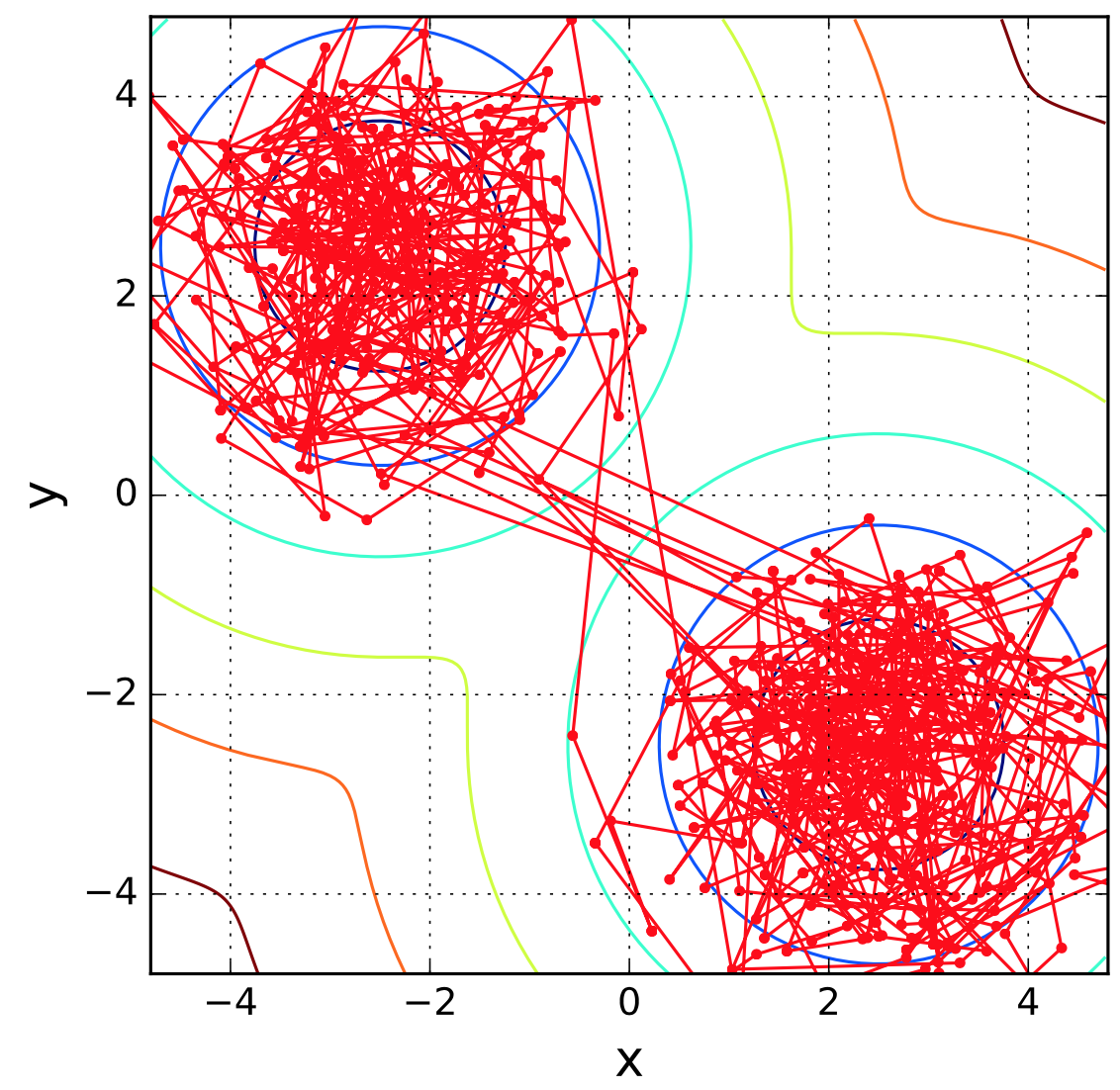}
			\label{samples}
		\end{minipage}%
		~
		\begin{minipage}[t]{0.32\textwidth}
			\centering
			\includegraphics[width=1\linewidth, height=1\linewidth]{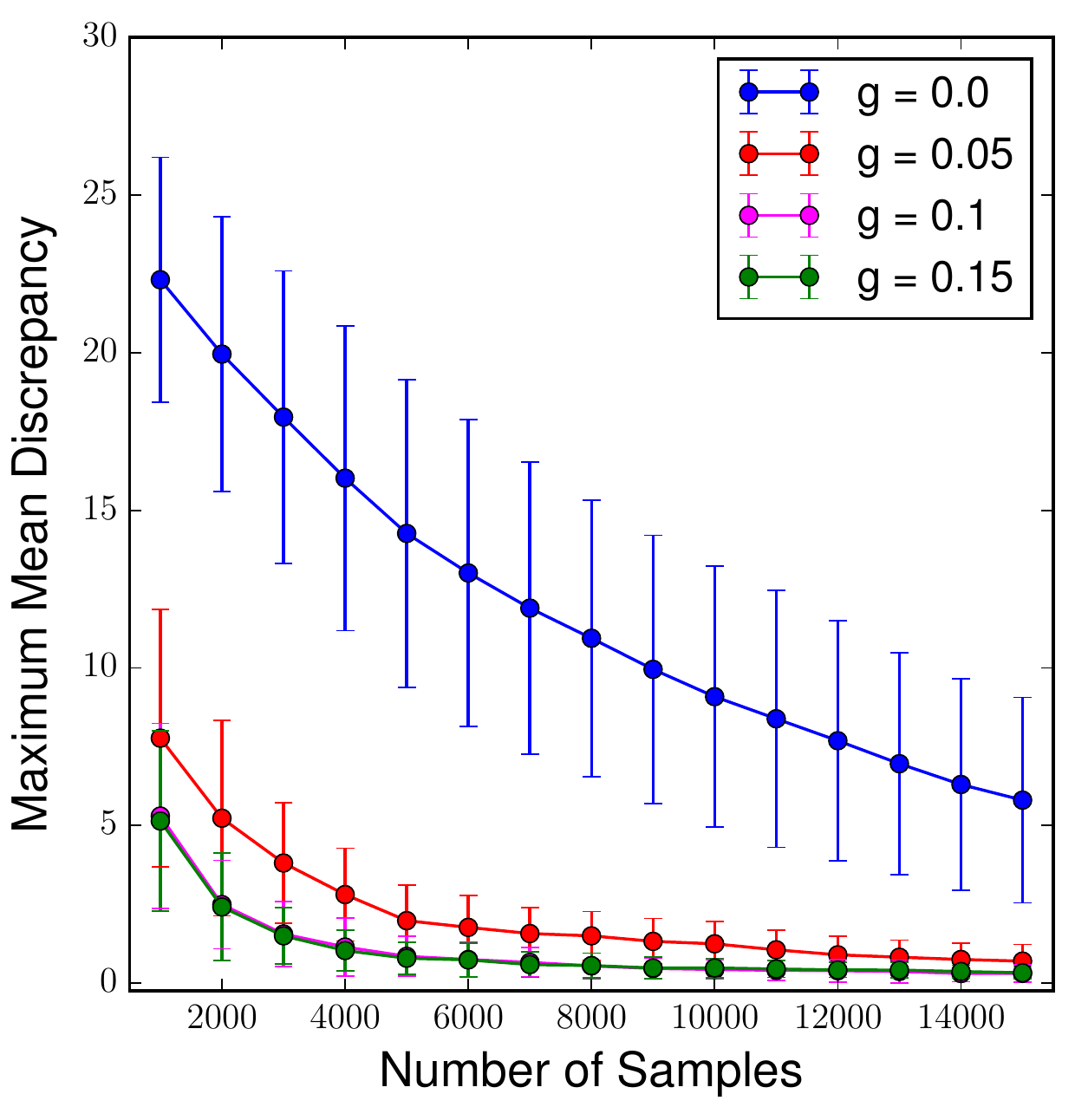}
			\label{mmd}
		\end{minipage}
		\caption{Left: 500 samples from HMC; Middle: 500 samples from MHMC; Right: MMD between HMC/MHMC samples for various magnitudes of the non-zero component of the magnetic field -- denoted g. Note $g=0$ corresponds to standard HMC.}
		\vspace{-.1cm}
		\label{fig:Mof2G}
	\end{figure*}
	We tuned HMC to achieve an acceptance rate of $\sim .75$ and used the same $\epsilon, L$ for MHMC, generating 15000 samples from both HMC and MHMC with these settings. The addition of the magnetic field term $\bottomrightmatrix$ -- which has one degree of freedom in this case -- introduces an asymmetric ``curl'' into the dynamics that pushes the sampler across the saddlepoint to the other mode allowing it to efficiently mix around both modes and between them -- see Figure \ref{fig:Mof2G}. The maximum mean discrepancy between exact samples generated from the target density and samples generated from both HMC and MHMC chains was also estimated for various magnitudes of $\bottomrightmatrix$, using a quadratic kernel $k(\bfs{x},\bfs{x}^\prime) = (1+\langle \bfs{x}, \bfs{x}^\prime \rangle)^2$ and averaged over 100 runs of the Markov chains \cite{Borgwardt2006}. Here, we clearly see that for various values of the nonzero component of $\bottomrightmatrix$, denoted $g$, the samples generated by MHMC more faithfully reflect the structure of the posterior. As before, we ran 50 parallel chains for $10^7$ timesteps to compute both the bias and Monte Carlo standard errors (MCSE) of the estimators of the target moments as shown in Table \ref{fig:2Gaussian2d}.
	\begin{table}[!ht]
		\vspace{-.25cm}
		\centering
		\caption{Bias $\pm$ MCSE for 2D Mixture of Gaussians for HMC vs. MHMC with $g=0.1$. Note that $\mathbb{E}[x]=0$ and $\mathbb{E}[x^2]=7.25$.}
		\begin{tabular}{cccc} \\\toprule
			algorithm  & $x$ (Bias $\pm$ MCSE) & $x^2$ (Bias $\pm$ MCSE)  \\ \midrule
			HMC & .0132 $\pm$ .0644 & 0.00264 $\pm$ 0.0114 \\  \midrule
			MHMC & .00239 $\pm$ .012 & 0.000596 $\pm$ 0.00365 \\ \midrule
		\end{tabular}
		\vspace{-.25cm}
		\label{fig:2Gaussian2d}
	\end{table}
	Additional experiments over a range of $\epsilon$, $L$ (and corresponding acceptance rates) and details are included in the Appendix for this example, demonstrating similar behavior.

	\subsection{FitzHugh-Nagumo model}
	Finally, we consider the problem of Bayesian inference over the parameters of the FitzHugh-Nagumo model (a set of nonlinear ordinary differential equations, originally developed to model the behavior of axial spiking potentials in neurons) as in \cite{Ramsay07, Girolami2011}. The FitzHugh-Nagumo model is a dynamical system $(V(t), R(t))$ defined by the following coupled differential equations:
	\begin{align}
	\dot{V}(t) = c(V(t)-V(t)^3/3+R(t)) \nonumber \\
	\dot{R}(t) = -(V(t)-a+bR(t))/c
	\label{FHN}
	\end{align}

	We consider the problem where the initial conditions $(V(0), R(0))$ of the system \eqref{FHN} are known, and a set of noise-corrupted observations $(\widetilde{V}(t_n), \widetilde{R}(t_n))_{n=0}^N = (V_{a,b,c}(t_n) + \varepsilon^V_{n}, R_{a, b, c}(t_n) + \varepsilon_n^R)_{n=0}^N$ at discrete time points $0=t_0 < t_1<\cdots < t_N$, are available - note that we illustrate dependence of the trajectories on the model parameters explicitly via subscripts. It is not possible to recover the true parameter values of the model from these observations, but we can  obtain a posterior distribution over them by specifying a model for the observation noise and a prior distribution over the model parameters.

	Similar to \cite{Ramsay07, Girolami2011}, we assume that the observation noise variables $(\varepsilon^V_n)_{n=0}^N$ and $(\varepsilon^V_n)_{n=0}^N$ are iid $\mathcal{N}(0, 0.1^2)$, and take an independent $\mathcal{N}(0,1)$ prior over each parameter $a$, $b$, and $c$. This yields a posterior distribution of the form
	\begin{align}
	& p(a,b,c) \propto \mathcal{N}(a; 0, 1)\mathcal{N}(b; 0, 1)\mathcal{N}(c; 0, 1) \ \times \nonumber \\
	& \prod_{n=0}^N \mathcal{N}(\widetilde{V}(t_n); V_{a,b,c}(t_n), 0.1^2)
	\label{eq:FHN-posterior}
	\end{align}
	Importantly, the highly non-linear dependence of the trajectory on the parameters $a$, $b$ and $c$ yields a complex posterior distribution - see Figure \ref{fig:FHN-plot}. Full details of the model set-up can be found in \cite{Ramsay07, Girolami2011}.

	For our experiments, we used fixed parameter settings of $a=0.2, b=0.2, c=3.0$ to generate 200 evenly-spaced noise-corrupted observations over the time interval $t = [0, 20]$ (as in \cite{Ramsay07, Girolami2011}). We performed inference over the posterior distribution of parameters $(a, b, c)$ with this set of observations using both the HMC and MHMC algorithms, which was perturbed with a magnetic field in each of the 3 axial planes of parameters -- along the $ab$, $ac$, and $bc$ axes with magnitude $g=0.1$. The chains were run to generate 1000 samples over 100 repetitions with settings of $\epsilon = 0.015, L=10$, which resulted in an average acceptance rate of $\sim .8$. The effective sample size of each of the chains normalized per unit time was then computed for each chain.
	\begin{wrapfigure}{r}{0.25\textwidth}
		\includegraphics[keepaspectratio, width=0.25\textwidth]{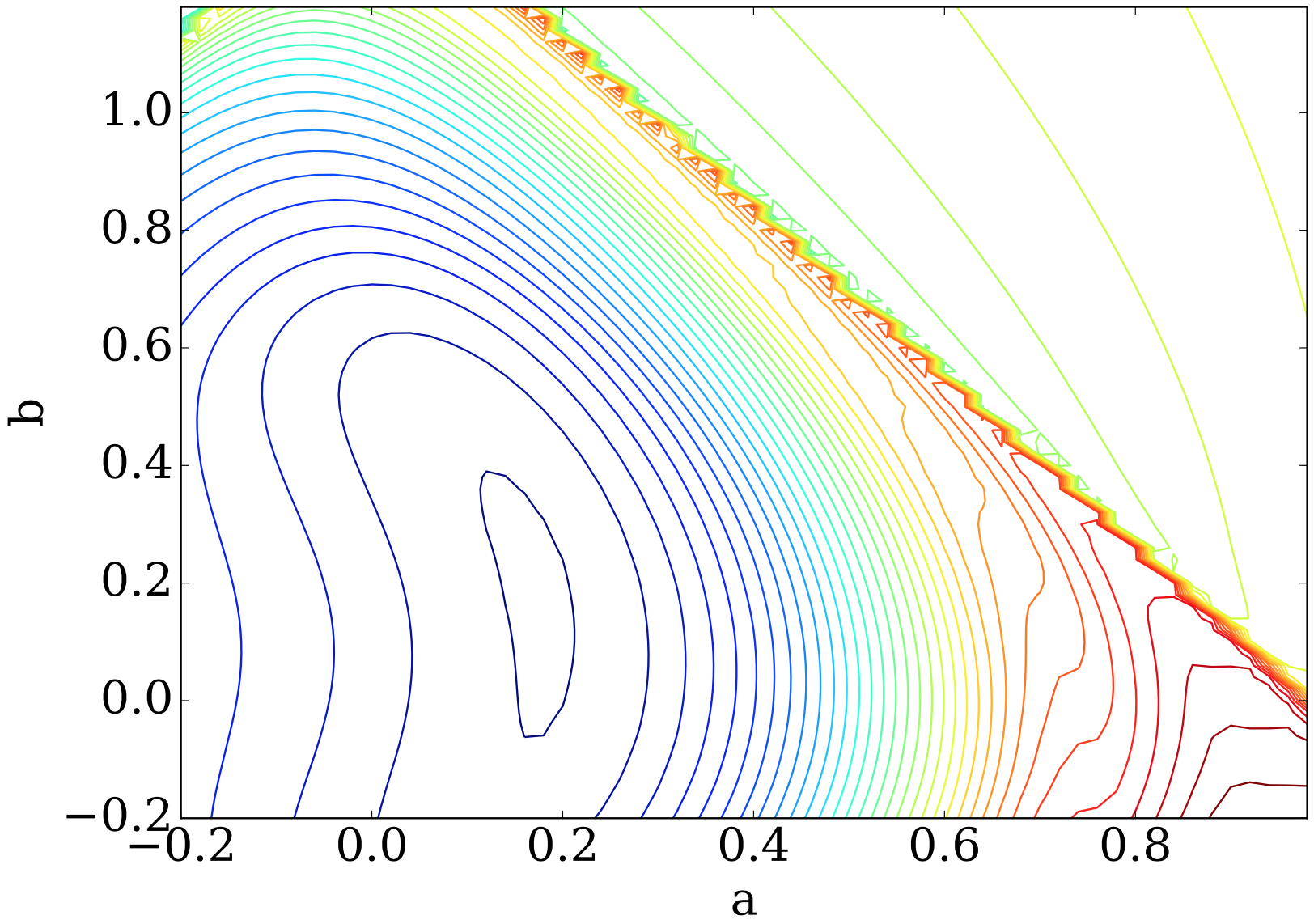}
		\caption{Marginal posterior density contour plot over $a$, $b$, with $c=3$.}
		\label{fig:FHN-plot}
		\vspace{-.5cm}
	\end{wrapfigure}

	Since each query to the posterior log-likelihood or posterior gradient log-likelihood requires solving an augmented set of differential equations as in \eqref{FHN}, the computation time ($\sim 238 \text{s}$) of all the methods was nearly identical.
	\begin{table}[!ht]
		\centering
		\caption{HMC vs. MHMC performance targeting the Fitzhugh-Nagumo posterior parameters}
		\begin{tabular}{cccc} \\\toprule
			algorithm  & ESS &ESS($a$, $b$, $c$)/time (s)  \\ \midrule
			HMC & 1000 318 606 & 4.20,  1.33,   2.55 \\  \midrule
			MHMC ab & 1000 349 658 &  4.20, 1.47, 2.77 \\  \midrule
			MHMC ac & 1000 336 628  &  4.20, 1.42, 2.64 \\  \midrule
			MHMC bc & 1000 326 649  &  4.20, 1.37, 2.73 \\  \midrule
		\end{tabular}
		\label{fig:Fitzhugh}
	\end{table}
	Moreover, note that all methods achieved nearly perfect mixing over the first coordinate so their effective sample size were truncated at 1000 for the $a$ coordinate. In this example, we can see that all magnetic perturbations slightly increase the mixing rate of the sampler over each of the $(b, c)$ coordinates with the $ab$ perturbation performing best.

	\section{Discussion and Conclusion}
	We have investigated a framework for MCMC algorithms based on non-canonical Hamiltonian dynamics and have given a construction for an explicit, symplectic integrator that is used to implement a generalization of HMC we refer to as magnetic HMC. We have also shown several examples where the non-canonical dynamics of MHMC can improve upon the sampling performance of standard HMC. Important directions for further research include finding more automated, adaptive mechanisms to set the matrix $\bfs{G}$, as well as investigating positionally-dependent magnetic field components, similar to how Riemannian HMC corresponds to local preconditioning. We believe that exploiting more general deterministic flows (such as also maintaining a non-zero $\bfs{E}$ in the top left-block of a general $\bfs{A}$ matrix) could form a fruitful area for further research on MCMC methods.

	\section*{Acknowledgements}
	The authors thank John Aston, Adrian Weller, Maria Lomeli, Yarin Gal and the anonymous reviewers for helpful comments. MR acknowledges support by the UK Engineering and Physical Sciences
	Research Council (EPSRC) grant EP/L016516/1 for the University of Cambridge
	Centre for Doctoral Training, the Cambridge Centre for Analysis. RET thanks EPSRC grants EP/M0269571 and EP/L000776/1 as well as Google for funding.

	\bibliography{mhmc}
	\bibliographystyle{icml2017}

	\appendix

	\date{}

	%


	\onecolumn
	\appendix

	\setcounter{section}{0}

	\section{Section 3 and 4 Proofs}
	\setcounter{lemma}{0}
	\renewcommand{\thelemma}{\arabic{lemma}}
	Here we provide proofs for results discussed in Section 3 of the main text regarding non-canonical dynamics.
	\begin{lemma}
	\label{lem:nchd-1}
		The map $\bfs{\Phi}_{\tau, H}^{\bfs{A}}(\theta, \bfs{p})$ defined by integrating the non-canonical Hamiltonian system
		\begin{align}
		\rms{\frac{d}{dt}} \begin{bmatrix} \bfs{\theta}(t) \\  \bfs{p}(t) \end{bmatrix}  = \bfs{A} \nabla_{\theta, \bfs{p}} H(\theta(t), \bfs{p}(t)) \label{eq:ncd}
		\end{align}
		with initial conditions $(\theta, \bfs{p})$ for time $\tau$, where $\bfs{A} \in \mathcal{M}_{2n\times2n}$ is \textit{any} invertible, antisymmetric matrix induces a flow on the coordinates $(\bfs{\theta}, \bfs{p})$ that is still \textit{energy-conserving} ($\partial_{\tau} H(\bfs{\Phi}_{\tau, H}^{\bfs{A}}(\theta, \bfs{p})) = 0$) and \textit{symplectic} with respect to $\bfs{A}$ ($[\nabla_{\theta, \bfs{p}} \bfs{\Phi}_{\tau, H}(\theta, \bfs{p})]^\top \bfs{A}^{-1} [\nabla_{\theta, \bfs{p}} \bfs{\Phi}_{\tau, H}(\theta, \bfs{p})] = \bfs{A}^{-1}$) which also implies volume-preservation of the flow.
	\end{lemma}
	\begin{proof} The proofs of both results simply uses the antisymmetry of $\bfs{A}$.

		\textbf{Energy-Conservation} -- Simply, we have that:
		\begin{align}
		& \partial_{\tau} H(\bfs{\Phi}_{\tau, H}(\theta, \bfs{p})) = \nabla_{\theta, \bfs{p}} H(\bfs{\Phi}_{\tau, H}(\theta, \bfs{p})) \partial_{\tau} \bfs{\Phi}_{\tau, H}(\theta, \bfs{p}) = \\
		& \nabla_{\theta, \bfs{p}} H(\bfs{\Phi}_{\tau, H}(\theta, \bfs{p}))^\top \bfs{A} \nabla_{\theta, \bfs{p}} H(\bfs{\Phi}_{\tau, H}(\theta, \bfs{p})) = 0
		\end{align}
		using the antisymmetry of $\bfs{A}$ and symmetry of $\nabla_{\theta, \bfs{p}} H(\bfs{\Phi}_{\tau, H}(\theta, \bfs{p})) \nabla_{\theta, \bfs{p}} H(\bfs{\Phi}_{\tau, H}(\theta, \bfs{p}))^\top$.

		\textbf{Symplecticness} -- We must show that the Jacobian of the flow generated by the dynamics preserves the non-canonical structure matrix $\bfs{A}$, which amounts to showing:
		\begin{align}
		\underbrace{[\nabla_{\theta, \bfs{p}} \bfs{\Phi}_{\tau, H}(\theta, \bfs{p})]^\top}_{F(\tau)^\top} \bfs{A}^{-1} \underbrace{[\nabla_{\theta, \bfs{p}} \bfs{\Phi}_{\tau, H}(\theta, \bfs{p})]}_{F(\tau)} = \bfs{A}^{-1}
		\end{align}
		where we define $F(\tau) = \nabla_{\theta, \bfs{p}} \bfs{\Phi}_{\tau, H}(\theta, \bfs{p})$ as the time-evolving Jacobian of the flow. First, note that $F(\tau)$ can be equivalently described as the solution to the differential equation:
		\begin{align}
		\frac{\rms{d}}{\rms{d\tau}} F(\tau) = \bfs{A} \nabla_{\theta, \bfs{p}}  \nabla_{\theta, \bfs{p}} H(\bfs{\Phi}_{\tau, H}(\theta, \bfs{p})) F(\tau)
		\end{align}
		with the initial condition $F(0) = \bfs{I}_{2d}$ (the Jacobian for the identity map at $t=0$). Trivially, we have:
		\begin{align}
		F(0) \bfs{A}^{-1} F(0) = \bfs{A}^{-1}
		\end{align}
		Then note that:
		\begin{align}
		& \frac{\rms{d}}{\rms{d\tau}}( F(\tau)^\top \bfs{A}^{-1} F(\tau) ) = \nonumber \\
		& F(\tau)^\top  \bfs{A}^{-1} \bfs{A} \nabla_{\theta, \bfs{p}}  \nabla_{\theta, \bfs{p}} H(\bfs{\Phi}_{\tau, H}(\theta, \bfs{p})) F(\tau) + F(\tau)^\top \nabla_{\theta, \bfs{p}}  \nabla_{\theta, \bfs{p}} H(\bfs{\Phi}_{\tau, H}(\theta, \bfs{p})) \bfs{A}^\top \bfs{A}^{-1} F(\tau) \nonumber \\
		& = F(\tau)^\top \nabla_{\theta, \bfs{p}}  \nabla_{\theta, \bfs{p}} H(\bfs{\Phi}_{\tau, H}(\theta, \bfs{p})) F(\tau) - F(\tau)^\top \nabla_{\theta, \bfs{p}}  \nabla_{\theta, \bfs{p}} H(\bfs{\Phi}_{\tau, H}(\theta, \bfs{p})) F(\tau) = \bfs{0} \nonumber
		\end{align}
		as desired by simply using $\bfs{A}^\top = - \bfs{A}$.
	\end{proof}

	\textbf{Time-Reversibility} --
	However, crucially it is \textit{not} the case that the Hamilton equations are time-reversible in the traditional sense of canonical Hamiltonian dynamics.

	\begin{lemma}\label{lem:nchd-2}
		If  $(\theta(t), \bfs{p}(t))$ is a solution to the non-canonical dynamics:
		\begin{align}
		\rms{\frac{d}{dt}} \begin{bmatrix} \bfs{\theta}(t) \\  \bfs{p}(t) \end{bmatrix}  = \underbrace{\begin{bmatrix} \topleftmatrix & \offdiagmatrix \\
			-\offdiagmatrix^\top & \bottomrightmatrix \end{bmatrix}}_{\bfs{A}} \begin{bmatrix} \nabla_{\bfs{\theta}} H(\theta(t), \bfs{p}(t)) \\  \nabla_{\bfs{p}} H(\theta(t), \bfs{p}(t))  \end{bmatrix} \label{eq:timereverse}
		\end{align}
		then $(\widetilde{\theta}(t), \widetilde{\bfs{p}}(t)) = (\theta(-t), -\bfs{p}(-t))$ is a solution to the modified non-canonical dynamics:
		\begin{align}
		\rms{\frac{d}{dt}} \begin{bmatrix} \widetilde{\bfs{\theta}}(t) \\  \widetilde{\bfs{p}}(t) \end{bmatrix}  = \underbrace{\begin{bmatrix} -\topleftmatrix & \offdiagmatrix \\
			-\offdiagmatrix^\top & -\bottomrightmatrix \end{bmatrix}}_{\bfs{\widetilde{A}}} \begin{bmatrix} \nabla_{\widetilde{\bfs{\theta}}} H(\widetilde{\theta}(t), \bfs{p}(t)) \\  \nabla_{\widetilde{\bfs{p}}} H(\widetilde{\theta}(t), \widetilde{\bfs{p}}(t)) \end{bmatrix} \label{timereverse2}
		\end{align}
		if $H(\bfs{\theta}, \bfs{p}) = H(\bfs{\theta}, -\bfs{p})$. In particular if $\topleftmatrix=\bottomrightmatrix=0$ then $\bfs{A}=\bfs{\widetilde{A}}$, which reduces to the traditional time-reversal symmetry of canonical Hamiltonian dynamics.
	\end{lemma}

	\begin{proof} A direct calculation yields
		\begin{align}
		\rms{\frac{d}{dt}} \begin{bmatrix} \tilde{\theta}(t) \\  \tilde{\bfs{p}}(t) \end{bmatrix} = \begin{bmatrix}  -\rms{\frac{d}{dt}} \bfs{\theta}(-t) \\   \rms{\frac{d}{dt}} \bfs{p}(-t) \end{bmatrix} = \begin{bmatrix}  -\bfs{E} \nabla_{\bfs{\theta}} H(\theta(-t)) - \bfs{F} \nabla_{\bfs{p}} H(\theta(-t)) \\ -\bfs{F}^\top \nabla_{\bfs{\theta}} H(\theta(-t)) + \bfs{G} \nabla_{\bfs{p}} H(\theta(-t)) \end{bmatrix} \nonumber \\ =
		\begin{bmatrix}  -\bfs{E} \nabla_{\tilde{\theta}} H(\tilde{\theta}(t))+ \bfs{F} \nabla_{\tilde{\bfs{p}}} H(\tilde{\theta}(t)) \\ -\bfs{F}^\top \nabla_{\tilde{\theta}} H(\tilde{\theta}(t))- \bfs{G} \nabla_{\tilde{\bfs{p}}} H(\tilde{\theta}(t)) \end{bmatrix} =  \underbrace{\begin{bmatrix} -\bfs{E} & \bfs{F} \\
			-\bfs{F}^\top & -\bfs{G} \end{bmatrix}}_{\bfs{\tilde{A}}}  \begin{bmatrix} \nabla_{\tilde{\theta}} H(\tilde{\theta}(t)) \\  \nabla_{\tilde{\bfs{p}}} H(\tilde{\theta}(t)) \end{bmatrix} \nonumber
		\end{align}
		\qedhere
	\end{proof}

	\subsection{Non-Canonical Dynamics Variable Augmentation} \label{sec:ncd-varaug}

	As remarked in the main text it is necessary to flip the $\bfs{E}$ and $\bfs{G}$ matrices at the end of a deterministic simulation of the Hamiltonian dynamics in order to render the proposal time-reversible which is in turn necessary to satisfy detailed balance. This is crucial for the correctness of the algorithm especially when an approximate simulation of the dynamics is used (as is always often the case).

	In particular, say that we wish to use $\bfs{\Phi}_{\tau, H}^{\bfs{A}}(\theta, \bfs{p})$ as a transition kernel with fixed, non-zero values of $\bfs{E} = \bfs{E}_0$ and $\bfs{G} = \bfs{G}_0$. We first augment the state-space by placing a symmetric, binary distribution independently over $\bfs{E}$ and $\bfs{G}$ such that $p(\bfs{E} = \bfs{E}_0) = p(\bfs{E} = -\bfs{E}_0)=1/2$ and $p(\bfs{G} = \bfs{G}_0) = p(\bfs{G} = -\bfs{G}_0)=1/2$, independently of $\theta, \bfs{p}$:
	\begin{align}
	\rho(\bfs{\theta}, \bfs{p}, \bfs{E}, \bfs{G}) \propto e^{-H(\theta, \bfs{p})} p(\bfs{E}) p(\bfs{G}). \label{eq:superjoint}
	\end{align}
	Importantly, this augmentation leaves the distribution over $\theta, \bfs{p}$ intact when $\bfs{E}$ and $\bfs{G}$ are marginalized out. Just as applying the momentum flip operator, $\bfs{\Phi}_{\bfs{p}}: (\theta, \bfs{p}, \bfs{E}, \bfs{G}) \to (\theta, -\bfs{p}, \bfs{E}, \bfs{G}) $, is a deterministic, energy-preserving, volume-preserving transformation, the $\bfs{E}$, $\bfs{G}$ flip operators, $\bfs{\Phi}_{\bfs{E}}: (\theta, \bfs{p}, \bfs{E}, \bfs{G}) \to (\theta, \bfs{p}, -\bfs{E}, \bfs{G}) $ and $\bfs{\Phi}_{\bfs{G}} :
	(\theta, \bfs{p}, \bfs{E}, \bfs{G}) \to (\theta, \bfs{p}, \bfs{E}, -\bfs{G})$, are also deterministic, energy-preserving, volume-preserving transformations that leave \eqref{eq:superjoint} invariant for this particular augmentation with $p(\bfs{E})$ and $p(\bfs{G})$. We can now build a self-inverse operator $\bfs{\widetilde{\Phi}}_{\tau, H}^{\bfs{A}}(\theta, \bfs{p})$, composed of simulating the Hamiltonian flow as $\bfs{\Phi}_{\tau, H}^{\bfs{A}}(\theta, \bfs{p})$ plus $ \bfs{\Phi}_{\bfs{E}} \circ \bfs{\Phi}_{\bfs{G}} \circ \bfs{\Phi}_{\bfs{p}}$, a flip of $\bfs{p}$, $\bfs{E}$, $\bfs{G}$, as:
	\begin{align}
	\bfs{\widetilde{\Phi}}_{\tau, H}^{\bfs{A}}(\theta, \bfs{p}) = \bfs{\Phi}_{\bfs{E}} \circ \bfs{\Phi}_{\bfs{G}} \circ \bfs{\Phi}_{\bfs{p}} \circ \bfs{\Phi}_{\tau, H}^{\bfs{A}}(\theta, \bfs{p})
	\end{align}
	Now we have constructed a deterministic, self-inverse map $\bfs{\widetilde{\Phi}}_{\tau, H}^{\bfs{A}}(\theta, \bfs{p})$. $\bfs{\widetilde{\Phi}}_{\tau, H}^{\bfs{A}}(\theta, \bfs{p})$ can now be used as the proposal for a reversible MCMC algorithm.

	An important point to note is that our choice of variable augmentation strategy, namely augmenting with binary distribution, is certainly not unique. However, it is perhaps the most natural and simplest choice which avoids the repetitive computation of matrix exponentials/diagonalizations since the approximate flow detailed in Section \ref{sec:sympintmag} will only need to compute matrix exponentials once upfront for $\pm \bfs{G}$.

	\subsection{Mass Preconditioning Proofs}  \label{sec:masspre}
	A common variation on standard HMC dynamics is to set the kinetic energy term in the Hamiltonian $H(\theta, \bfs{p})$ to $\frac{1}{2}\bfs{p}^\top \bfs{M}^{-1} \bfs{p}$ for some symmetric positive-definite matrix $\bfs{M}$, and sample the initial momentum variable $\bfs{p}$ from the corresponding distribution $\mathcal{N}(0,\bfs{M})$. However, we can contextualize preconditioning using a non-canonical $\bfs{A}$ matrix in the following manner:
	\begin{lemma}\label{lem:preconditioned}
		i) Preconditioned HMC with momentum variable $\bfs{p} \sim \mathcal{N}(0, \bfs{M})$ in the $(\theta, \bfs{p})$ coordinates, is exactly equivalent to simulating non-canonical HMC with $\bfs{p}' = \bfs{M}^{-1/2} \bfs{p} \sim \mathcal{N}(0, \bfs{I})$ and the non-canonical matrix:
		\[
		\bfs{A} = \begin{bmatrix} \bfs{0} & \bfs{M}^{1/2}  \\ -(\bfs{M}^{1/2})^\top & \bfs{0} \end{bmatrix}
		\]
		and then transforming back to $(\theta, \bfs{p})$ coordinates using $\bfs{p} = \bfs{M}^{1/2} \bfs{p}'$. Here $\bfs{M}^{1/2}$ is a Cholesky factor for $\bfs{M}$.

		ii) On the other hand if we apply a change of basis (via an invertible matrix $\bfs{F}$) to our coordinates $\theta' = \bfs{F}^{-1} \theta$, simulate HMC in the $(\theta', \bfs{p})$ coordinates, and transform back to the original basis using $\bfs{F}$, this is exactly equivalent to non-canonical HMC with
		\[
		\bfs{A} = \begin{bmatrix} \bfs{0} & \bfs{F} \\ -\bfs{F}^\top & \bfs{0}  \end{bmatrix}
		\]
	\end{lemma}

	\begin{proof}
		We first prove the equivalence regarding the change of basis in momentum space. Under the $\bfs{M}$ mass matrix variant of HMC, $\bfs{p}$ is drawn from a $\mathcal{N}(\bfs{0}, \bfs{M})$ distribution, and the dynamics of $\theta, \bfs{p}$ are then given by
		\[
		\rms{\frac{d}{dt}} \begin{bmatrix} \bfs{\theta} \\ \bfs{p} \end{bmatrix} = \begin{bmatrix} \bfs{M}^{-1}\bfs{p}  \\  -\nabla_\theta U(\theta) \end{bmatrix}
		\]
		Denoting the upper-triangular Cholesky factor of $\bfs{M}^{-1}$ by $\bfs{M}^{-1/2}$, and introducing a new variable $\bfs{p}^\prime = \bfs{M}^{-1/2} \bfs{p}$, we obtain the following dynamics for the joint variable $(\theta, \bfs{p}^\prime)$:
		\[
		\rms{\frac{d}{dt}} \begin{bmatrix} \bfs{\theta} \\ \bfs{p}^\prime \end{bmatrix} = \rms{\frac{d}{dt}} \begin{bmatrix} \bfs{\theta} \\ \bfs{M}^{-1/2}\bfs{p} \end{bmatrix} = \begin{bmatrix} (\bfs{M}^{-1/2})^\top\bfs{p}^\prime  \\  -\bfs{M}^{-1/2} \nabla_\theta U(\theta) \end{bmatrix} = \begin{bmatrix} \bfs{0} & (\bfs{M}^{-1/2})^\top \\ -\bfs{M}^{-1/2} & \bfs{0} \end{bmatrix} \begin{bmatrix} \nabla_\theta U(\theta) \\ \bfs{p}^\prime  \end{bmatrix}
		\]
		Further, note that if the marginal distribution of $\bfs{p}$ is $\mathcal{N}(0, \bfs{M})$, then under this change of variables $\bfs{p}^\prime$ has the marginal distribution $\mathcal{N}(0, \bfs{I})$. Thus, simulating canonical HMC with a non-identity mass matrix is equivalent to making a change of basis in momentum space, simulating non-canonical HMC with a particular choice of non-canonical $\bfs{A}$ matrix, and finally reverting back to the original basis.

		We now prove the equivalence regarding the change of basis in $\theta$ space, which follows similarly. Consider non-canonical HMC on the state-momentum pair $(\theta, \bfs{p})$, with the antisymmetric matrix $\bfs{A}$ taking the particular form
		\[
		\bfs{A} = \begin{pmatrix} \bfs{0} & \bfs{F} \\ -\bfs{F}^\top & \bfs{0} \end{pmatrix}
		\]
		The states $\theta, \bfs{p}$ obtained from this algorithm are equal to those obtained by first changing basis to $\theta'= \bfs{F}^{-1} \theta$, then simulating standard HMC dynamics for the pair $(\theta', \bfs{p})$ with respect to the Hamiltonian
		\begin{align*}
		H'(\theta', \bfs{p}) & = U'(\theta') + \frac{1}{2}\bfs{p}^\top \bfs{p} \\
		& = U(\bfs{F} \theta) + \frac{1}{2}\bfs{p}^\top \bfs{p}
		\end{align*}
		and then reverting to the original basis as $\theta = \bfs{F} \theta^\prime$. To see this, first note that if we denote the distribution on the coordinates $\theta$ corresponding to the potential $U$ by $\pi(\theta) = e^{-U(\theta)}$, then the corresponding distribution on the coordinates $\theta'$ is given by $\pi'$, where
		\[
		\pi'(\theta') = \mathrm{det}(\bfs{F}) \pi(\bfs{F}\theta')
		\]
		The corresponding potential $U'$ is therefore given by
		\[
		U'(\theta') = U(\bfs{F}\theta')
		\]
		Running canonical HMC dynamics targeting the Hamiltonian $H'$ yields the dynamics:
		\[
		\rms{\frac{d}{dt}} \begin{bmatrix} \theta' \\ \bfs{p} \end{bmatrix} = \begin{bmatrix} \bfs{p}  \\  -\nabla_{\theta'} U(\theta') \end{bmatrix}
		\]
		But note then that the dynamics of the original coordinates are then given by:
		\begin{align}
		& \rms{\frac{d}{dt}} \begin{bmatrix} \theta \\ \bfs{p} \end{bmatrix} = \rms{\frac{d}{dt}} \begin{bmatrix} \bfs{F} \theta' \\ \bfs{p} \end{bmatrix} = \begin{bmatrix} \bfs{F} \bfs{p}  \\  -\nabla_{\theta'} U(\theta') \end{bmatrix} = \begin{bmatrix} \bfs{F} \bfs{p}  \\  -\nabla_{\theta'} U(\bfs{F} \theta) \end{bmatrix} = \begin{bmatrix} \bfs{F} \bfs{p}  \\  -(\bfs{F}^\top) \nabla_{\theta} U(\theta)   \end{bmatrix} \nonumber \\
		& = \begin{bmatrix} \bfs{0} & \bfs{F} \\ -\bfs{F}^\top & \bfs{0} \end{bmatrix} \begin{bmatrix} \nabla_{\theta} U(\theta) \\ \bfs{p} \end{bmatrix} \nonumber
		\end{align}
		which are exactly a special case of non-canonical HMC dynamics described above.
		\qedhere
	\end{proof}

	\section{Magnetic HMC (MHMC)}
	Here we provide proofs related to the dynamics of magnetic HMC and it's symplectic integration scheme.
	\subsection{Non-Canonical Dynamics and Magnetism}
	We first establish the connection between the particular subcase of non-canonical Hamiltonian dynamics where
	\[
	\bfs{A} =
	\begin{bmatrix} \bfs{0} & \bfs{I} \\
	-\bfs{I} & \bfs{G} \end{bmatrix}
	\]

	and Newton's law for a charged particle coupled to a magnetic field.
	\begin{lemma}\label{lem:magnetic}
		In 3-dimensions the non-canonical Hamiltonian dynamics, with Hamiltonian $H(\theta, \bfs{p}) = U(\bfs{\theta}) + \frac{1}{2} \bfs{p}^\top \bfs{p}$, correspond to simulating the differential equations:
		\begin{align}
		\rms{\frac{d}{dt}} \begin{bmatrix} \bfs{\theta} \\  \bfs{p} \end{bmatrix}  = \underbrace{\begin{bmatrix} \bfs{0} & \bfs{I} \\
			-\bfs{I} & \bfs{G} \end{bmatrix}}_{\bfs{A}} \begin{bmatrix} \nabla_{\bfs{\theta}} H \\  \nabla_{\bfs{p}} H  \end{bmatrix} \equiv  \underbrace{\begin{bmatrix} 0 & \bfs{I} \\
			-\bfs{I} & \bfs{G} \end{bmatrix}}_{\bfs{A}} \begin{bmatrix} \nabla_{\bfs{\theta}} U(\bfs{\theta}) \\  \bfs{p} \end{bmatrix}
		\label{eq:magnetic3d}
		\end{align}
		where
		\[
		\bfs{G} = \begin{bmatrix} 0 & -b_3 & b_2 \\ b_3 & 0 & -b_1 \\ -b_2 & b_1 & 0 \end{bmatrix}
		\]
		are equivalent to the Newtonian mechanics of a charged particle (with unit mass and charge) coupled to a magnetic field $\vec{\bfs{B}} = \begin{bmatrix} b_1  \\ b_2 \\ b_3  \end{bmatrix}$ which take the form:
		\begin{align}
		\frac{\rms{d}^2 \theta}{\rms{dt^2}} = -\nabla_{\bfs{\theta}}U(\bfs{\theta})+ \frac{\rms{d} \theta}{\rms{dt}} \times \vec{\bfs{B}} &
		\end{align}
		where $\theta$ is simply a 3-dimensional vector and $\times$ the cross-product.
	\end{lemma}

	\begin{proof}
		If we let $\theta$ and $\bfs{p}$ denote our position and momentum coordinates in 3 dimensions then Newton's law for a charged particle in a magnetic field (with $m=q=1$) is:
		\begin{align}
		\frac{\rms{d}^2 \theta}{\rms{dt^2}} = -\nabla_{\theta} U(\theta) + \frac{\rms{d} \theta}{\rms{dt}} \times \vec{\bfs{B}} &
		\end{align}
		Defining momentum canonically as $\frac{\rms{d} \theta}{\rms{dt}} = \bfs{p}$ we have:
		\begin{align}
		\rms{\frac{d}{dt}} \begin{bmatrix} \theta \\  \bfs{p} \end{bmatrix}  = \begin{bmatrix} \bfs{p} \\  -\nabla_{\theta} U(\theta) + \frac{\rms{d} \theta}{\rms{dt}} \times \vec{\bfs{B}} \end{bmatrix} = \begin{bmatrix} \bfs{p} \\  -\nabla_{\theta} U(\theta) + \bfs{G} \bfs{p} \end{bmatrix}
		\equiv  \underbrace{\begin{bmatrix} \bfs{0} & \bfs{I} \\
			-\bfs{I} & \bfs{G} \end{bmatrix}}_{\bfs{A}} \begin{bmatrix} \nabla_{\theta} U(\theta) \\  \bfs{p} \end{bmatrix}
		\label{eq:magnetic3d}
		\end{align}
	\end{proof}

	We now show that the dynamics used in magnetic HMC cannot be reproduced by simply choosing a different smooth Hamiltonian, $H'(\theta, \bfs{p})$ and using the canonical $\bfs{A}$ matrix:
	\[
	\bfs{A} =
	\begin{bmatrix} \bfs{0} & \bfs{I} \\
	-\bfs{I} & \bfs{0} \end{bmatrix}
	\]
	to generate the dynamics.
	\begin{lemma}\label{lem:impossible}
		The non-canonical Hamiltonian dynamics with magnetic $\bfs{A}$ and Hamiltonian $H(\bfs{\theta}, \bfs{p}) =  U(\bfs{\theta}) + \frac{1}{2} \bfs{p}^\top \bfs{p}$ cannot be obtained using canonical Hamiltonian dynamics for any choice of smooth Hamiltonian.
	\end{lemma}
	\begin{proof}
		Consider the ODEs corresponding to non-canonical dynamics with magnetic $\bfs{A}$ and $H(\theta, \bfs{p}) =  U(\bfs{\theta}) + \frac{1}{2}\bfs{p}^\top \bfs{p}$:
		\begin{align}
		\rms{\frac{d}{dt}} \begin{bmatrix} \bfs{\theta} \\  \bfs{p} \end{bmatrix}  = \begin{bmatrix} \bfs{0} & \bfs{I} \\
		\bfs{-I} & \bfs{G} \end{bmatrix} \begin{bmatrix} \nabla_{\theta} U(\theta) \\  \bfs{p} \end{bmatrix} = \begin{bmatrix} \bfs{p}  \\ -\nabla_{\theta} U(\theta) + \bfs{G} \bfs{p} \end{bmatrix}.
		\end{align}
		Assume, to obtain a contradiction, that these canonical Hamiltonian dynamics can be reproduced for some choice of smooth $H'(\theta, \bfs{p})$ and canonical $\bfs{A}$ matrix (i.e. $\bfs{E} = \bfs{G} = \bfs{0}$, $\bfs{F} = \bfs{I}$):
		\begin{align}
		\rms{\frac{d}{dt}} \begin{bmatrix} \bfs{\theta} \\  \bfs{p} \end{bmatrix}  = \begin{bmatrix} \bfs{0} & \bfs{I} \\
		\bfs{-I} & \bfs{0} \end{bmatrix} \begin{bmatrix} \nabla_{\theta} H'(\theta, \bfs{p}) \\  \nabla_{\bfs{p}} H'(\theta, \bfs{p}) \end{bmatrix} = \begin{bmatrix} \bfs{p}  \\ -\nabla_{\theta} U(\theta) + \bfs{G} \bfs{p} \end{bmatrix}.
		\end{align}
		This implies:
		\begin{align}
		\begin{bmatrix} \nabla_{\bfs{p}} H'(\theta, \bfs{p}) \\  \nabla_{\theta} H'(\theta, \bfs{p}) \end{bmatrix} = \begin{bmatrix} \bfs{p} \\  \nabla_{\theta} U(\theta) - \bfs{G} \bfs{p}
		\end{bmatrix}
		\implies
		\begin{bmatrix} \nabla_{\theta} \nabla_{\bfs{p}} H'(\theta, \bfs{p}) \\  \nabla_{\bfs{p}} \nabla_{\theta} H'(\theta, \bfs{p}) \end{bmatrix} = \begin{bmatrix} \bfs{0} \\ -\bfs{G} \end{bmatrix}.
		\end{align}
		However, as long as the 2nd-order mixed partial derivatives are continuous they must be equal; so the conclusion follows.
	\end{proof}
	\subsection{Symplectic Integrator for Magnetic Dynamics} \label{sec:sympintmag}
	We begin by considering the symmetric splitting:
	\begin{align}
	H(\theta, \bfs{p}) = \underbrace{U(\bfs{\theta})/2}_{H_1(\rms{\theta})} + \underbrace{\bfs{p}^T\bfs{p}/2}_{H_2(\bfs{p})} +  \underbrace{U(\bfs{\theta})/2}_{H_1(\rms{\theta})} \label{eq:split}
	\end{align}
	The corresponding non-canonical dynamics for the sub-Hamiltonians $H_1(\theta)$ and $H_2(\bfs{p})$ are:
	\begin{align}
	\rms{\frac{d}{dt}} \begin{bmatrix} \bfs{\theta} \\  \bfs{p} \end{bmatrix}  = \underbrace{\begin{bmatrix} \bfs{E} & \bfs{F} \\
		\bfs{-F}^\top & \bfs{G} \end{bmatrix}}_{\bfs{A}} \begin{bmatrix} \nabla_{\bfs{\theta}}U(\bfs{\theta})/2 \\  \bfs{0} \end{bmatrix} = \begin{bmatrix} \bfs{E} \nabla_{\bfs{\theta}}U(\bfs{\theta})/2  \\  \bfs{-F}^\top \nabla_{\bfs{\theta}}U(\bfs{\theta})/2 \end{bmatrix} \label{eq:nchd-flow1}
	\end{align}
	and
	\begin{align}
	\rms{\frac{d}{dt}} \begin{bmatrix} \bfs{\theta} \\  \bfs{p} \end{bmatrix}  = \underbrace{\begin{bmatrix} \bfs{E} & \bfs{F} \\
		\bfs{-F}^\top & \bfs{G} \end{bmatrix}}_{\bfs{A}} \begin{bmatrix} \bfs{0} \\  \bfs{p} \end{bmatrix} = \begin{bmatrix} \bfs{F} \bfs{p}  \\  \bfs{G} \bfs{p} \end{bmatrix}. \label{eq:nchd-flow2}
	\end{align}
	We denote the corresponding flows by $\bfs{\Phi}_{\epsilon, H_1(\theta)}^\bfs{A}$ and $\bfs{\Phi}_{\epsilon, H_2(\bfs{p})}^\bfs{A}$ respectively. The flow in \eqref{eq:nchd-flow1} is generally not explicitly tractable unless we take $\bfs{E}= 0$ -- in which case it is solved by an Euler translation as for standard Hamiltonian dynamics.

	Crucially, the flow in \eqref{eq:nchd-flow2} is a \textit{linear} differential equation and hence analytically integrable.
	Without loss of generality, we restrict ourselves to the case $\bfs{F} = \bfs{I}$ (the case for general $\bfs{F}$ follows similarly). The dynamics associated with the flow $H_2(\bfs{p})$ introduced in Lemma \ref{lem:magnetic} become
	\[
	\rms{\frac{d}{dt}} \begin{bmatrix} \bfs{\theta}(t) \\ \bfs{p}(t) \end{bmatrix} = \begin{bmatrix} \bfs{p}(t)  \\  \bfs{G} \bfs{p}(t) \end{bmatrix}
	\]
	with initial condition $(\bfs{\theta}_{0}, \bfs{p}_0)$. Using the power series representation of the matrix exponential, the second differential equation for $\bfs{p}$ may be integrated analytically to yield the following flow in $\bfs{p}$-space:
	\[
	\bfs{p}(t) = \exp(\bfs{G}t) \bfs{p}_{0}
	\]
	Substituting this result into the differential equation for $\theta$ yields
	\[
	\rms{\frac{d\theta}{dt}} = \exp(\bfs{G}t) \bfs{p}_0
	\]
	If $\bfs{G}$ is invertible then once again using the power series representation of the matrix exponential and rearranging yields the solution
	\[
	\bfs{\Phi}_{\epsilon, H_2(\bfs{p})}  \begin{bmatrix} \bfs{\theta} \\ \bfs{p} \end{bmatrix} = \begin{bmatrix} \bfs{\theta} + \bfs{G}^{-1}(\exp(\bfs{G}\epsilon) - \bfs{I})\bfs{p} \\ \exp(\bfs{G}\epsilon)\bfs{p} \end{bmatrix}
	\]
	If $\bfs{G}$ is not invertible, then slightly more care must be taken to first diagonalize $\bfs{G}$ and separate its invertible/singular components. Since $\bfs{G}$ is strictly antisymmetric it can be written as $i \bfs{H}$ where $\bfs{H}$ is a Hermitian matrix. Thus it  can be diagonalized over $\mathbb{C}$ as:

	\[
	\bfs{G} =
	\begin{bmatrix}
	U_{\bfs{\Lambda}} & U_{0}
	\end{bmatrix}
	\begin{bmatrix}
	\bfs{\Lambda} & \bfs{0} \\

	\bfs{0} & \bfs{0}
	\end{bmatrix}
	\begin{bmatrix}
	U_{\bfs{\Lambda}}^\top \\
	U_{0}^\top
	\end{bmatrix}
	\]
	where $\bfs{\Lambda}$ is a diagonal submatrix consisting of the nonzero eigenvalues of $\bfs{G}$. $
	\begin{bmatrix}
	U_{\bfs{\Lambda}} & U_{0}
	\end{bmatrix}
	$ and $
	\begin{bmatrix}
	U_{\bfs{\Lambda}}^\top \\
	U_{0}^\top
	\end{bmatrix}
	$ are unitary matrices where the columns of $U_{\Lambda}$ are the eigenvectors of $\bfs{G}$ corresponding to its nonzero eigenvalues while the columns of $U_{0}$ are the eigenvectors of $\bfs{G}$ corresponding to its zero eigenvalues. Even if $\bfs{G}$ is not invertible we still have:
	\[
	\bfs{p}(t) = \exp(\bfs{G}t) \bfs{p}_{0}
	\]
	However it is more convenient to represent the matrix exponential as:
	\[
	\exp(\bfs{G}t) =
	\begin{bmatrix}
	U_{\bfs{\Lambda}} & U_{0}
	\end{bmatrix}
	\begin{bmatrix}
	\exp(\bfs{\Lambda}t) & \bfs{0} \\
	\bfs{0} & \bfs{I}
	\end{bmatrix}
	\begin{bmatrix}
	U_{\bfs{\Lambda}}^\top \\
	U_{0}^\top
	\end{bmatrix}
	\]
	Substituting this result into the differential equation for $\theta$, this representation of $\exp(\bfs{G}t)$ implies the non-identity block can be handled as in the invertible case while the $\bfs{I}$ block follows trivially to give:
	\[
	\theta(t) = \theta_0 +
	\begin{bmatrix}
	U_{\bfs{\Lambda}} & U_{0}
	\end{bmatrix}
	\begin{bmatrix}
	\bfs{\Lambda}^{-1}(\exp(\bfs{\Lambda}t) - \bfs{I}) & \bfs{0} \\
	\bfs{0} & t\bfs{I}
	\end{bmatrix}
	\begin{bmatrix}
	U_{\bfs{\Lambda}}^\top \\
	U_{0}^\top
	\end{bmatrix}
	\bfs{p}_0
	\]
	Note that if $\bfs{G}=\bfs{0}$ then the flow map will simply reduce to an Euler translation as in ordinary HMC. We can also combine the ideas of Section \ref{sec:masspre} to obtain a preconditioned, magnetic HMC algorithm corresponding to a general $\bfs{A}$-matrix of the form
	\[
	\begin{pmatrix} \bfs{0} & \bfs{F} \\ -\bfs{F}^\top & \bfs{G} \end{pmatrix}
	\]

	Dealing with a non-zero $\bfs{E}$ becomes more subtle, since the corresponding sub-Hamiltonian is no longer exactly integrable under the splitting construction. In order to exactly integrate this sub-block a more costly implicit integrator is needed.

	\subsection{Integration Error of Magnetic HMC} \label{sec:interrormag}

	Since we are using a symmetric, leapfrog splitting scheme for magnetic HMC that \textit{exactly} integrates each sub-Hamiltonian we obtain identical error scaling to the leapfrog integrator applied to canonical HMC. Indeed, symplectic integrators are well-known to have many nice error properties in general and so perhaps this result is not so surprising \cite{Hairer2006}.
	\begin{lemma}\label{lem:error}
		The symplectic leapfrog-like integrator for magnetic HMC will have the same local ($\sim \mathcal{O}(\epsilon^3)$) and global ($\sim \mathcal{O}(\epsilon^2)$) error scaling (over $\tau \sim \frac{L}{\epsilon}$ steps), as the canonical leapfrog integrator of standard HMC if the Hamiltonian is separable.
	\end{lemma}
	\begin{proof}
		Note that for the parametrization of $\bfs{A}$ corresponding to magnetic HMC the Hamiltonian vector field $\vec{\bfs{H}} = \nabla_{\bfs{p}} H \nabla_{\theta} + (-\nabla_{\theta} + \bfs{G} \nabla_{\bfs{p}} H) \nabla_{\bfs{p}} \equiv \vec{\bfs{A}} + \vec{\bfs{B}}$ will generate the \textit{exact} flow corresponding to exactly simulating the dynamics. We obtain an $\mathcal{O}(\epsilon^3)$ local error by simply exploiting the separability of the Hamiltonian. The leapfrog integration scheme splits the Hamiltonian as: $H(\theta, \bfs{p}) = H_{1}(\theta) + H_{2}(\bfs{p}) + H_{1}(\theta)$ and \textit{exactly} integrates each sub-Hamiltonian so:
		\begin{align}
		& \bfs{\Phi}^{\mathrm{frog}}_{\epsilon, H} = \bfs{\Phi}_{\epsilon, H_1(\rms{\theta})}  \circ \bfs{\Phi}_{\epsilon, H_2(\bfs{p})}  \circ \bfs{\Phi}_{\epsilon, H_1(\rms{\theta})} = \exp \left( \frac{\epsilon}{2} \vec{\bfs{B}} \right) \circ  \exp \left( \epsilon \vec{\bfs{A}} \right) \circ  \exp \left( \frac{\epsilon}{2} \vec{\bfs{B}} \right) \\
		\end{align}
		Via repeated applications of the Baker-Campbell-Hausdorff formula \cite{Hairer2006} obtain:
		\begin{align}
		& \exp \left( \frac{\epsilon}{2} \vec{\bfs{B}} \right) \circ \exp \left( \epsilon \vec{\bfs{A}} + \frac{\epsilon}{2} \vec{\bfs{B}} + \frac{\epsilon^2}{2}[\vec{\bfs{A}}, \vec{\bfs{B}}] \right) + \mathcal{O}(\epsilon^3) = \\
		& \exp \left( \frac{\epsilon}{2} \vec{\bfs{B}} + \epsilon \vec{\bfs{A}} + \frac{\epsilon}{2} \vec{\bfs{B}} + \frac{\epsilon^2}{2} [\vec{\bfs{A}}, \vec{\bfs{B}}] + \frac{1}{2} [\frac{\epsilon}{2} \vec{\bfs{B}}, \epsilon \vec{\bfs{A}} + \frac{\epsilon}{2} \vec{\bfs{B}} + \frac{\epsilon^2}{2} [\vec{\bfs{A}}, \vec{\bfs{B}}]] \right) + \mathcal{O}(\epsilon^3) = \\
		& \exp \left( \epsilon \vec{\bfs{H}} + \frac{\epsilon^2}{4}[\vec{\bfs{A}}, \vec{\bfs{B}}] + \frac{\epsilon^2}{4}[\vec{\bfs{B}}, \vec{\bfs{A}}] + \frac{\epsilon^2}{8}[\vec{\bfs{B}}, \vec{\bfs{B}}] \right) + \mathcal{O}(\epsilon^3) = \exp \left( \epsilon \vec{\bfs{H}} \right )  + \mathcal{O}(\epsilon^3)
		\end{align}
		where we have used the antisymmetry of the commutator. The global error scaling, for an integration time of $\tau = \frac{L}{\epsilon}$  follows straightforwardly:
		\begin{align}
		& \bfs{\Phi}^{\mathrm{frog}}_{\tau, H} = \left( \exp \left(\frac{\epsilon}{2} \vec{\bfs{B}} \right) \circ  \exp \left( \epsilon \vec{\bfs{A}} \right) \circ  \exp \left( \frac{\epsilon}{2} \vec{\bfs{B}} \right) \right)^{L} \\
		& = \left( \exp \left(\epsilon \vec{\bfs{H}} \right) + \mathcal{O}(\epsilon^3) \right)^L \\
		& = \exp \left(\epsilon L \vec{\bfs{H}} \right) + L\epsilon \mathcal{O}(\epsilon^2) \\
		& = \exp \left( \tau \vec{\bfs{H}} \right) + \tau \mathcal{O}(\epsilon^2) \\
		& = \exp \left(\tau \vec{\bfs{H}} \right) + \mathcal{O}(\epsilon^2)
		\end{align}
		as desired.
	\end{proof}

	\section{Section 6 Experimental Details}
	Here we provide relevant experimental details for some of the Experiments presented in the main text.

	\subsection{Gaussians}

	In both experiments the reported autocorrelation measures are averaged over all coordinates as well as over 100 independent runs of the HMC/MHMC chains.

	\subsubsection{2D Gaussian}
	For the uncorrelated, ill-conditioned 2D Gaussian experiment presented in the main text the magnetic $\bfs{G}$ component only has one non-zero parameter which was set to $g=.2$.

	\subsubsection{10D Gaussian}
	For the uncorrelated, ill-conditioned 10D Gaussian experiment presented in the main text, the $\bfs{G}$ matrix was set to encourage the flow of momentum between the directions of large marginal variance with covariance eigenvalues $10^6$ and the remaining 8 directions of directions of small marginal variance with covariance eigenvalues of $1$. We denote the directions of large marginal variance as $x_1$, $x_2$, and the other 8 directions of directions of small marginal variance as $x_i$. $\bfs{G}$ was set such that $\bfs{G}_{1i} = \bfs{G}_{2i} = g$, $\bfs{G}_{i1} = \bfs{G}_{i2} = -g$ and $\bfs{G}_{12} = \bfs{G}_{21} = 0$ for $g=.2$.

	\subsection{Mixture of Gaussians}
	The superior mixing of MHMC relative to HMC in this example holds true for a wide range of $(\epsilon, L)$ settings as we can see by looking at the maximum mean discrepancy as a function of the number of samples in both Figures \ref{fig:mofg1} and \ref{fig:mofg2}.
	\begin{figure*}[!ht]
		\centering
		\begin{minipage}[t]{.3\textwidth}
			\centering
			\includegraphics[width=1\linewidth]{img/Mof2G/MMD/{MMDforMof2GBinary_B_eps_1.5_L_33-eps-converted-to}.pdf}
		\end{minipage}%
		~
		\begin{minipage}[t]{.3\textwidth}
			\centering
			\includegraphics[width=1\linewidth]{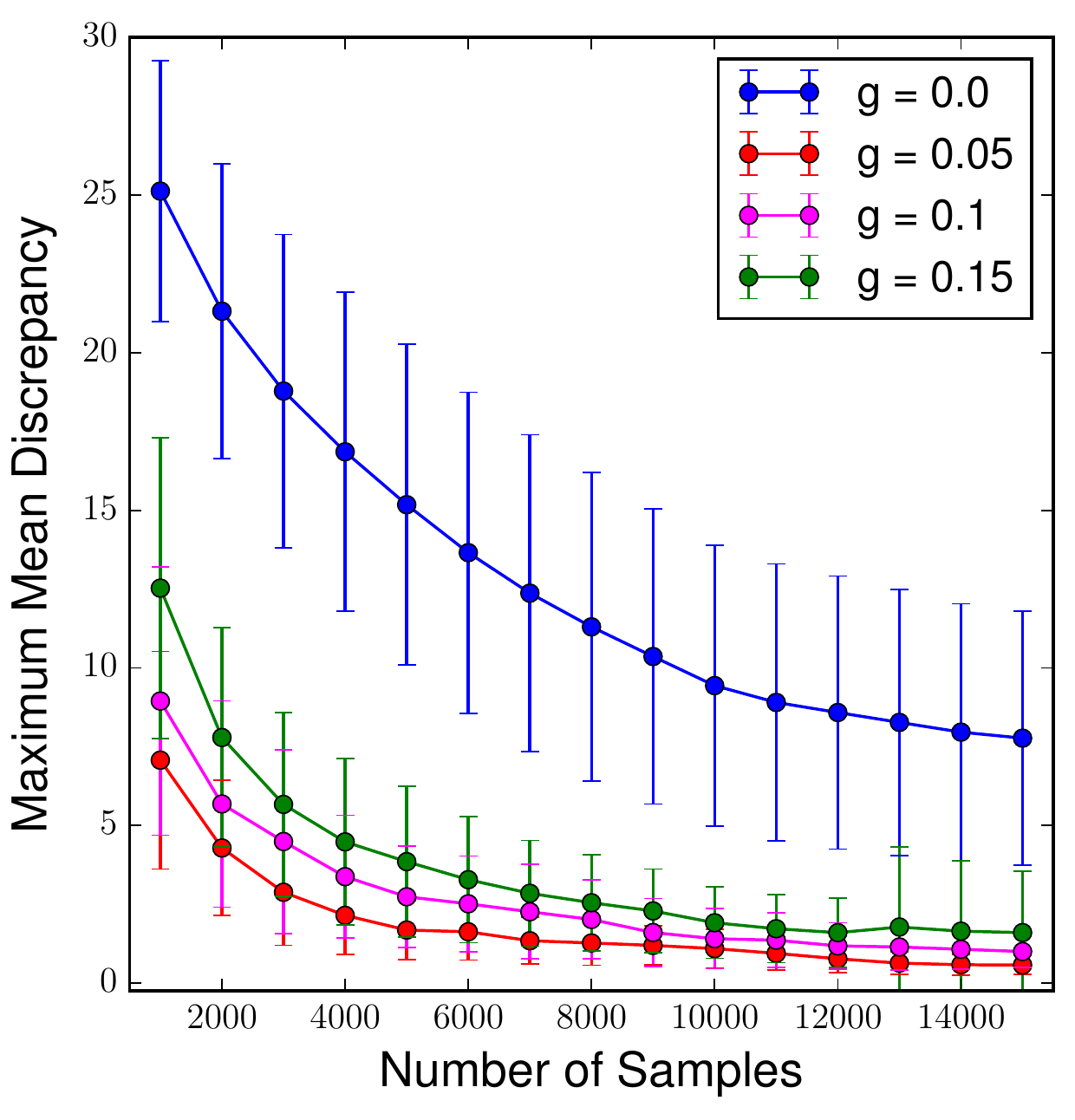}
		\end{minipage}%
		\caption{Left: MMD vs. Number of Samples for $(\epsilon =1.5, L=33)$. The acceptance rate was $\sim .74$ for all $g$. Right: MMD vs. Number of Samples for $(\epsilon =1.9, L=40)$. The acceptance rate was $\sim .43$ for all $g=0$ and $\sim .34$ for all non-zero $g$. In both diagrams $g$ denotes the non-zero component of the magnetic field.}
		\label{fig:mofg1}
	\end{figure*}
	\begin{figure*}[!ht]
		\centering
		\begin{minipage}[!ht]{.3\textwidth}
			\centering
			\includegraphics[width=1\linewidth]{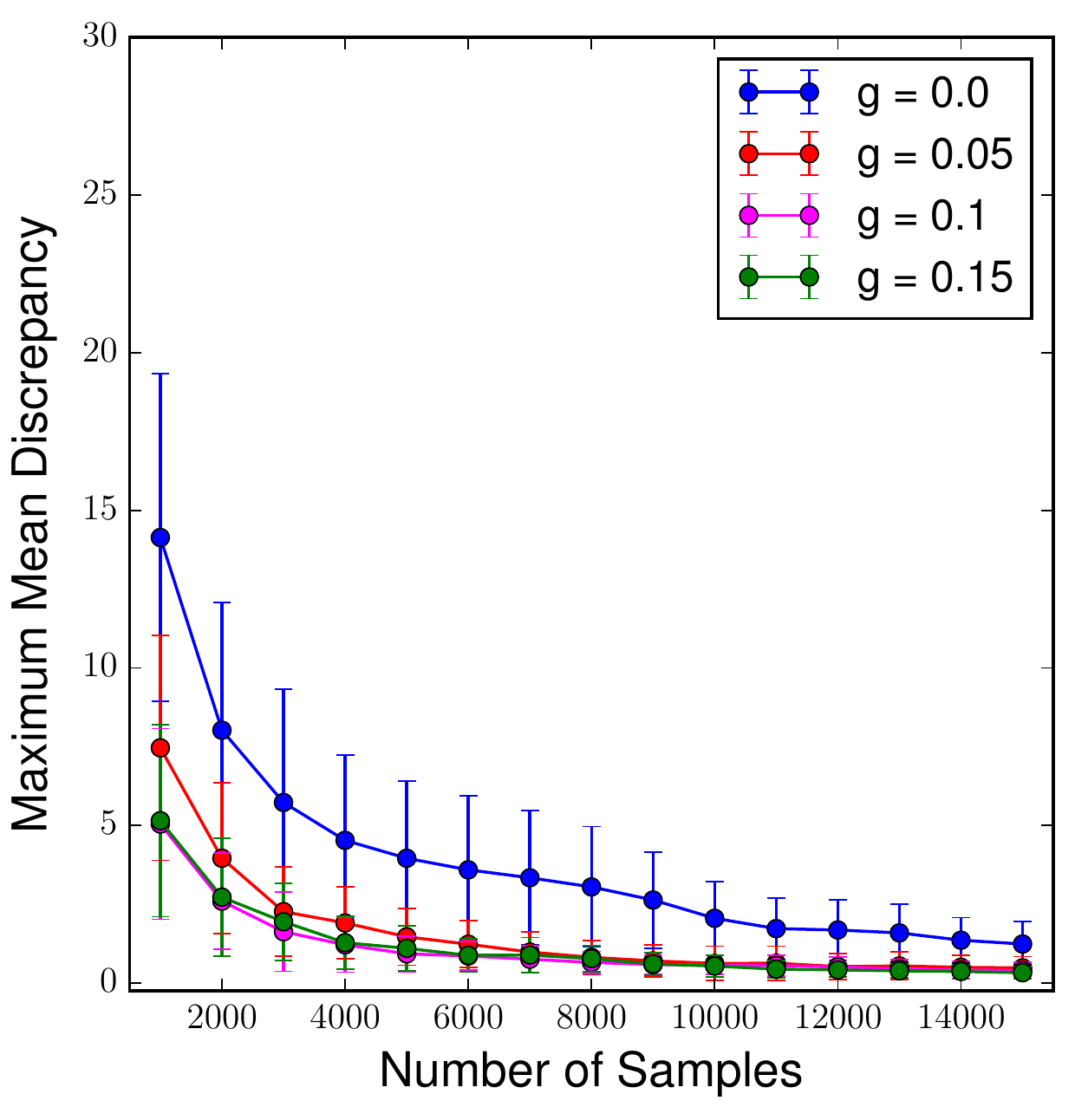}
		\end{minipage}%
		~
		\begin{minipage}[!ht]{.3\textwidth}
			\centering
			\includegraphics[width=1\linewidth]{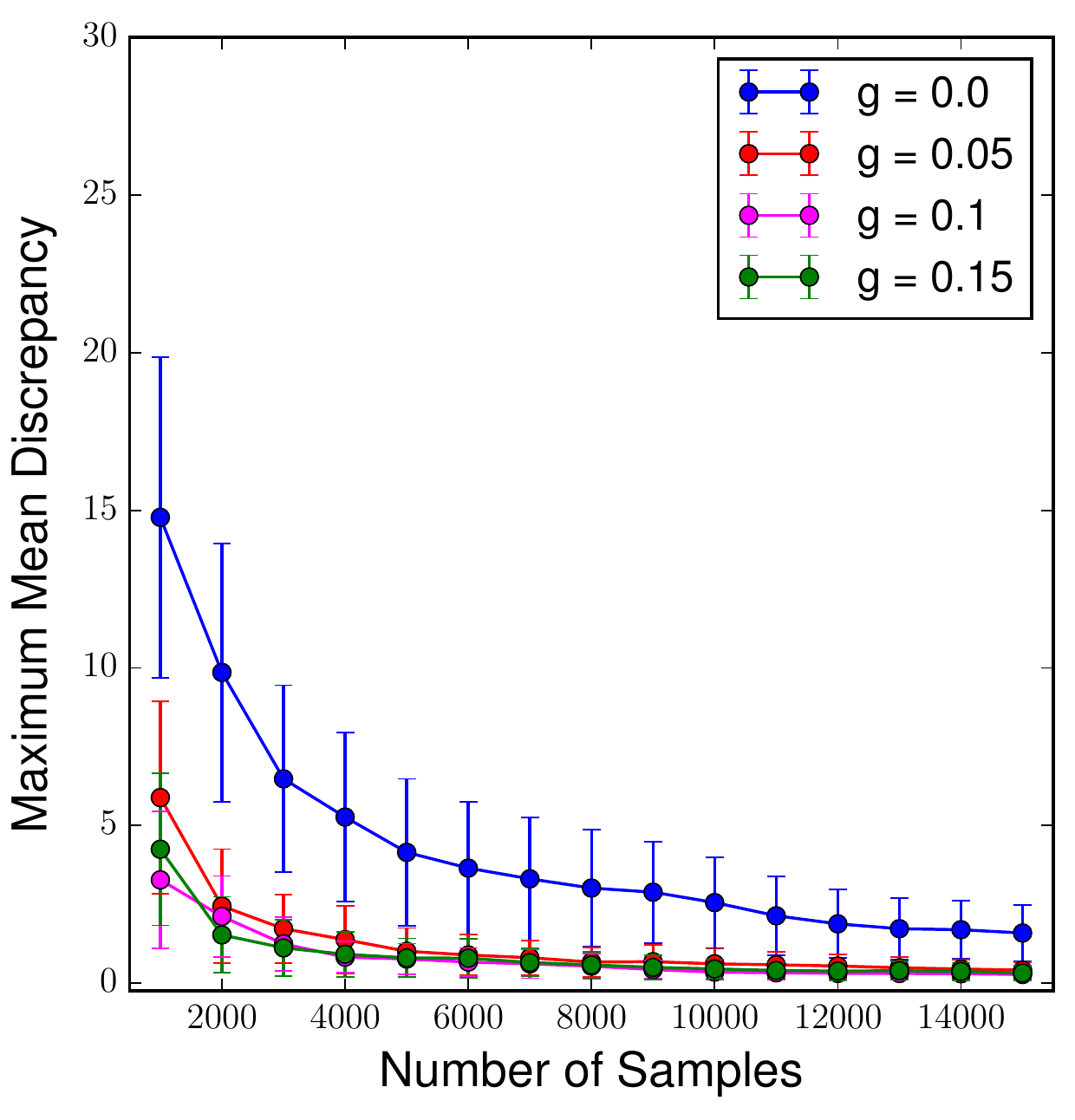}
		\end{minipage}%
		\caption{Left: MMD vs. Number of Samples for $(\epsilon=1.0, L=50)$. The acceptance rate was $\sim .87$ for all $g$. Right: MMD vs. Number of Samples for $(\epsilon =0.5, L=110)$. The acceptance rate was $\sim .95$ for all $g$. In both diagrams $g$ denotes the non-zero component of the magnetic field.}
		\label{fig:mofg2}
	\end{figure*}

	We found that tuning the parameters $(\epsilon, L)$ via Bayesian optimization often resulted in worse performance for ordinary HMC since the values found for $(\epsilon, L)$ were too conservative to encourage exploration between both modes. Moreover, more aggressive choices for $(\epsilon, L)$ for ordinary HMC led to a sharp drop in acceptance rate and significantly worse performance.

	\subsection{Gaussian Funnel}
	In this additional experiment, we consider the Gaussian funnel of \cite{Neal2003} with density
	$$p(\bfs{x}, v) = \Pi_{i=1}^{n} \mathcal{N}(x_{i} | 0, e^{-v}) \mathcal{N}(v | 0, 3^2)$$
	in 10+1 dimensions (i.e. $n=10$). This density illustrates the pathological correlation present in many hierarchical models between $\bfs{x}$, a vector of low-level parameters, and $v$, a hyperparameter controlling their variability. As noted in \cite{betancourt2015hamiltonian, Zhang2014}, Riemannian HMC methods, which incorporate local curvature information of the target, are well-suited to this problem as they help the dynamics traverse the energy surface which rapidly changes as a $v$ varies. HMC (as well as MHMC) do not exploit curvature information and will have more difficulty exploring the $v$ direction due to the rapid variation in density -- see \cite{betancourt2015hamiltonian} for a detailed discussion.
	Despite this difficulty, we might intuitively expect that introducing a ``curl'' term into the entries of $\bottomrightmatrix$ which couple each $x_i$ and $v$ could increase exploration of the dynamics since these variables are nonlinearly correlated.
	In order to encourage the periodic flow of momentum between the marginal direction $v$ and the coordinates $x_i$ the $\bfs{G}$ matrix was set such that $\bfs{G}_{vi} = g$, $\bfs{G}_{iv} = -g$, $\bfs{G}_{ij} = 0$ with $g=.2$.
	To investigate this, we generated 10000 samples from both HMC and MHMC, discarding 1000 burn-in samples and computed the minimum effective sample size across $\bfs{x}$ and $v$ and bias in the moments of the $v$ parameter similar to the set-up in \cite{Zhang2014} for various $\epsilon, L$ (see Table \ref{fig:gaussianFunnel}).
	We report results averaged over 100 different runs of the Markov chains.
	\begin{table*}[!htb]
		\centering
		\vspace{-.3cm}
		\caption{Comparison of HMC and MHMC targeting the Gaussian funnel for a variety of leapfrog steps}
		\resizebox{.8\textwidth}{!}{%
			\begin{tabular}{cccccc} \\\toprule
				algorithm & settings & time (s) & min ESS($\bfs{x}$, $v$)  & min ESS($\bfs{x}$, $v$)/s & MSE($\mathbb{E}[v]$, $\mathbb{E}[v^2]$) \\ \midrule
				HMC & $\epsilon = 0.05, L=100$ & 225  & 414, 85 & 1.84, 0.38 & .59, 1.57 \\  \midrule
				MHMC & $\epsilon = 0.05, L=100$ & 270 & 463, 97 & 1.71, 0.36 & .29, 1.17 \\  \midrule
				HMC & $\epsilon = 0.05, L=300$ & 705  & 1342, 118  & 1.90, 0.17 & .35, 1.15 \\  \midrule
				MHMC & $\epsilon = 0.05, L=300$ & 837 & 1554, 122 & 1.86,  0.15 & .15, 1.05 \\  \midrule
			\end{tabular} %
		}
		\label{fig:gaussianFunnel}
	\end{table*}

	We find that adding the magnetic field component decreases the bias in the moments and marginally increases the ESS, although both samplers struggle to explore the full target density, as the relatively low ESS figures indicate. Further details and experiments are provided in the Appendix.

	Recall the density of the Gaussian funnel $p(\bfs{x}, v) = \Pi_{i=1}^{n} \mathcal{N}(x_{i} | 0, e^{-v}) \mathcal{N}(v | 0, 3^2)$. In order to encourage the periodic flow of momentum between the marginal direction $v$ and the coordinates $x_i$ the $\bfs{G}$ matrix was set such that $\bfs{G}_{vi} = g$, $\bfs{G}_{iv} = -g$, $\bfs{G}_{ij} = 0$ with $g=.2$. Moreover the reported results were averaged over 100 different runs of the Markov chains.

	\section{MHMC Proposals and Dynamics}
	In this section, we provide illustrations of the proposal distributions of MHMC in simple low-dimensional settings, to aid intuition and demonstrate the divergence of its behaviour from standard HMC.

	\subsection{Gaussian Densities}
	We first consider the case of an isotropic Gaussian target, and illustrate the proposal distribution of standard HMC, as well as MHMC with a variety of settings for the skew-symmetric matrix $\mathbf{A} = \begin{pmatrix} \topleftmatrix & \offdiagmatrix \\ \offdiagmatrix & \bottomrightmatrix \end{pmatrix}$ - see Figure \ref{fig:gaussian-props}. As in previous sections, we denote the off-diagonal element of the $\mathbf{G}$ matrix by $g$. We also provide proposal plots for an anisotropic Gaussian target distribution - see Figure \ref{fig:aniso-gaussian-props}.
	\begin{figure}[!ht]
		\centering
		\subfigure[Standard HMC (g=0)]{
		\includegraphics[keepaspectratio, width=0.25\textwidth]{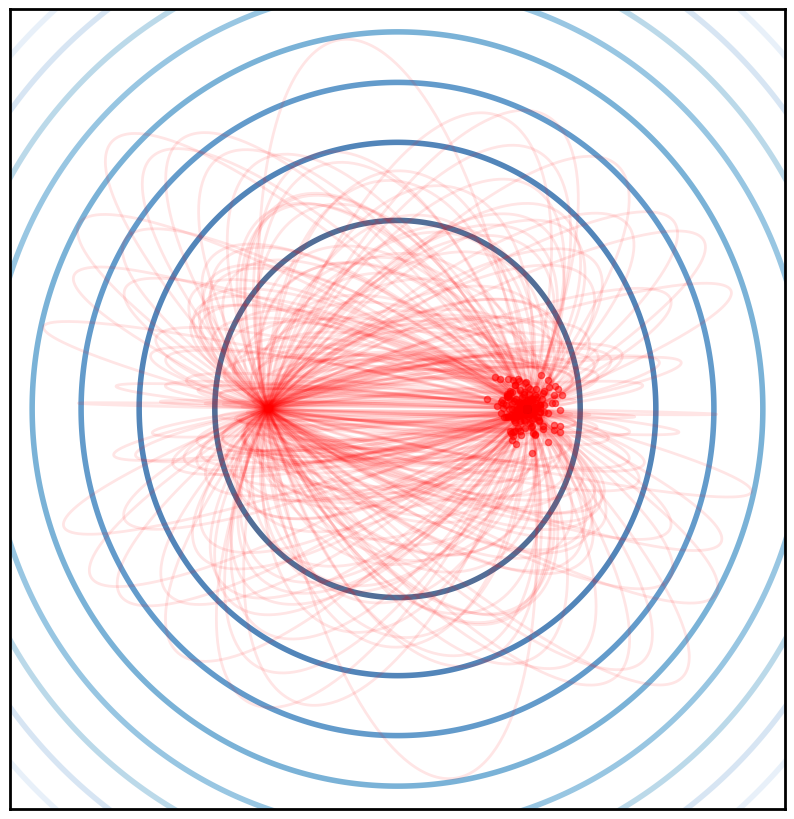}}%
		\subfigure[MHMC (g=0.5)]{
		\includegraphics[keepaspectratio, width=0.25\textwidth]{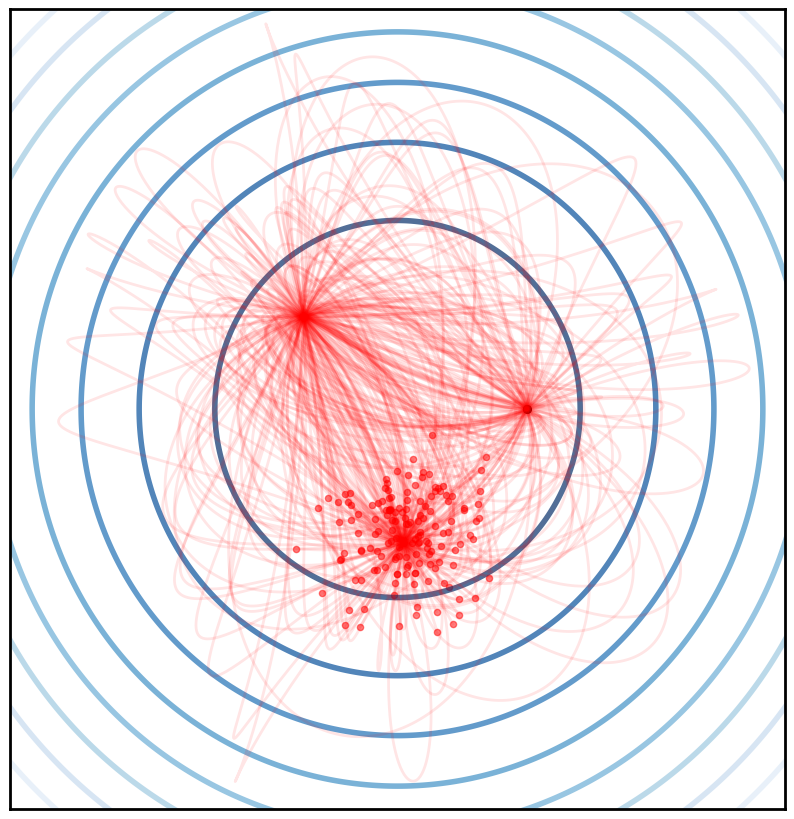}}%
		\subfigure[MHMC (g=1.0)]{
		\includegraphics[keepaspectratio, width=0.25\textwidth]{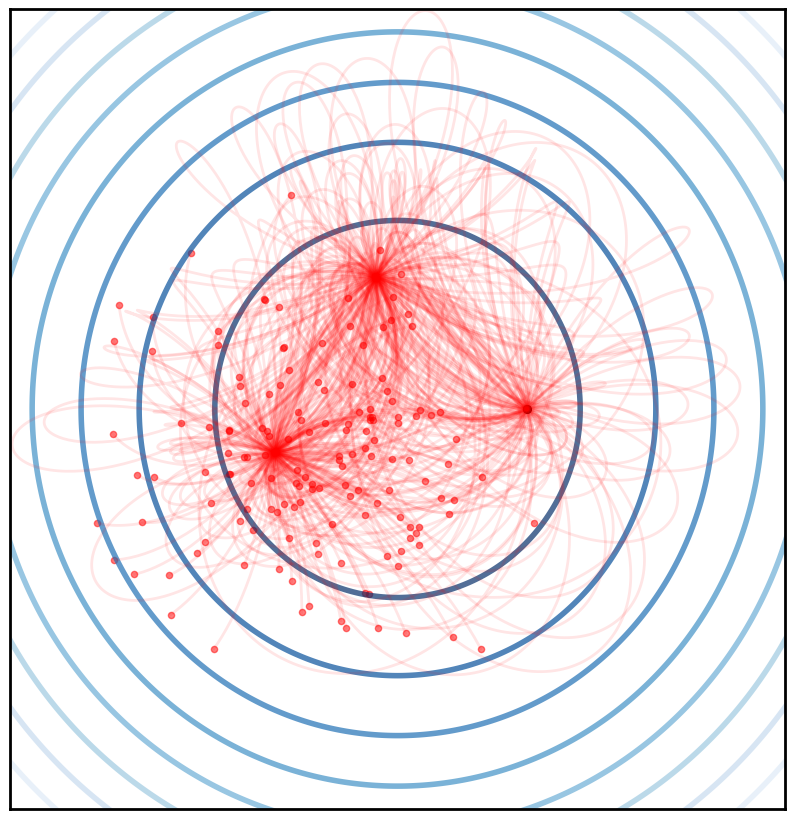}}%

		\subfigure[MHMC (g=2.0)]{
		\includegraphics[keepaspectratio, width=0.25\textwidth]{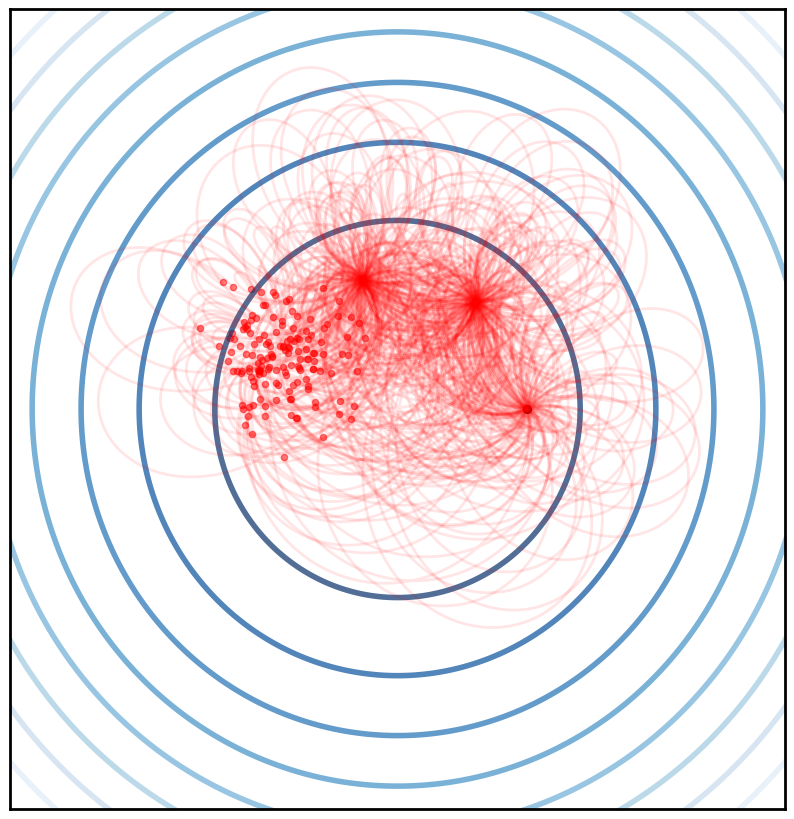}}%
		\subfigure[MHMC (g=3.0)]{
		\includegraphics[keepaspectratio, width=0.25\textwidth]{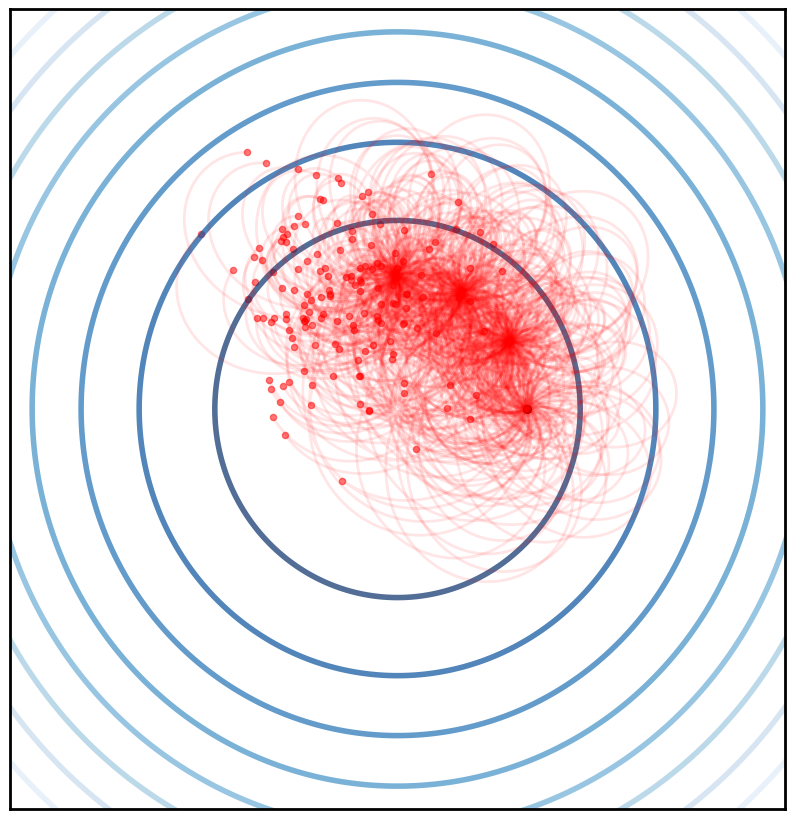}}%
		\subfigure[MHMC (g=4.0)]{
		\includegraphics[keepaspectratio, width=0.25\textwidth]{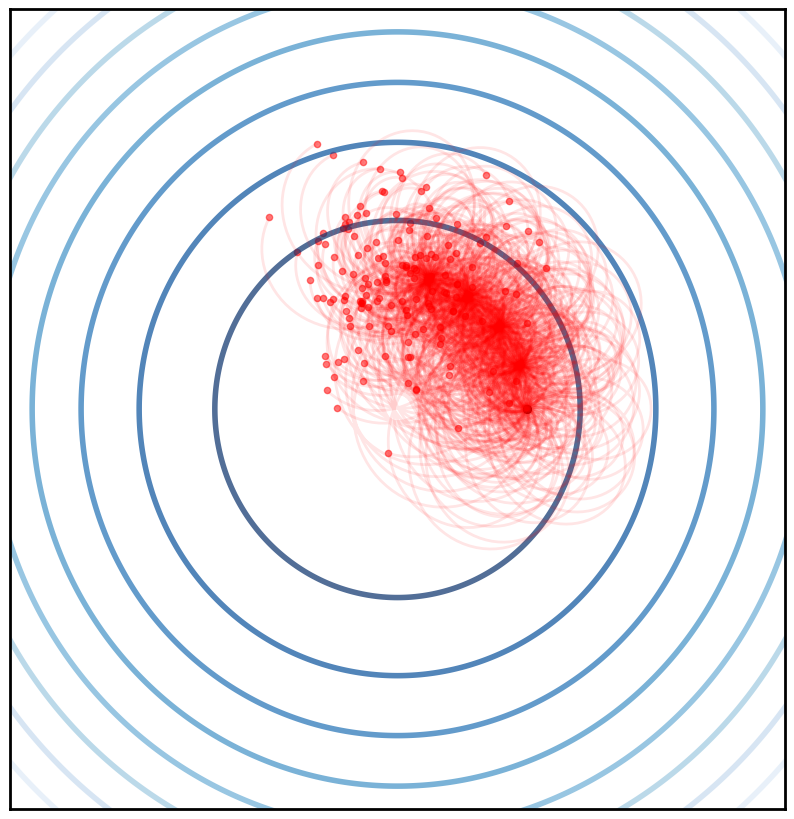}}%
		\caption{HMC and MHMC proposals for an isotropic Gaussian target.}
		\label{fig:gaussian-props}
	\end{figure}
	\begin{figure}[!ht]
		\centering
		\subfigure[Standard HMC (g=0)]{
			\includegraphics[keepaspectratio, width=0.25\textwidth]{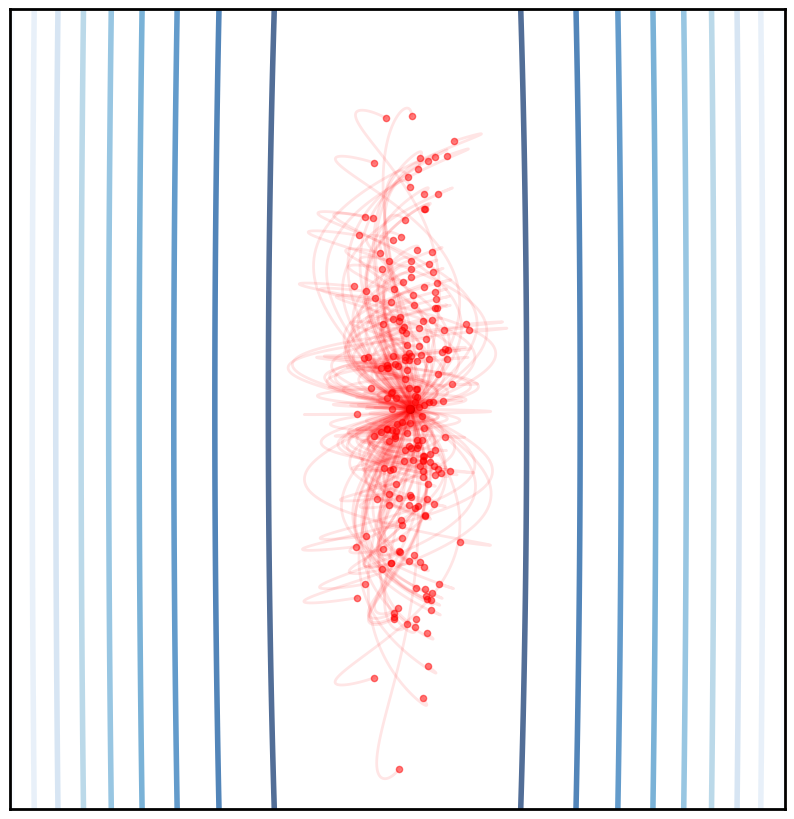}}%
		\subfigure[MHMC (g=0.5)]{
			\includegraphics[keepaspectratio, width=0.25\textwidth]{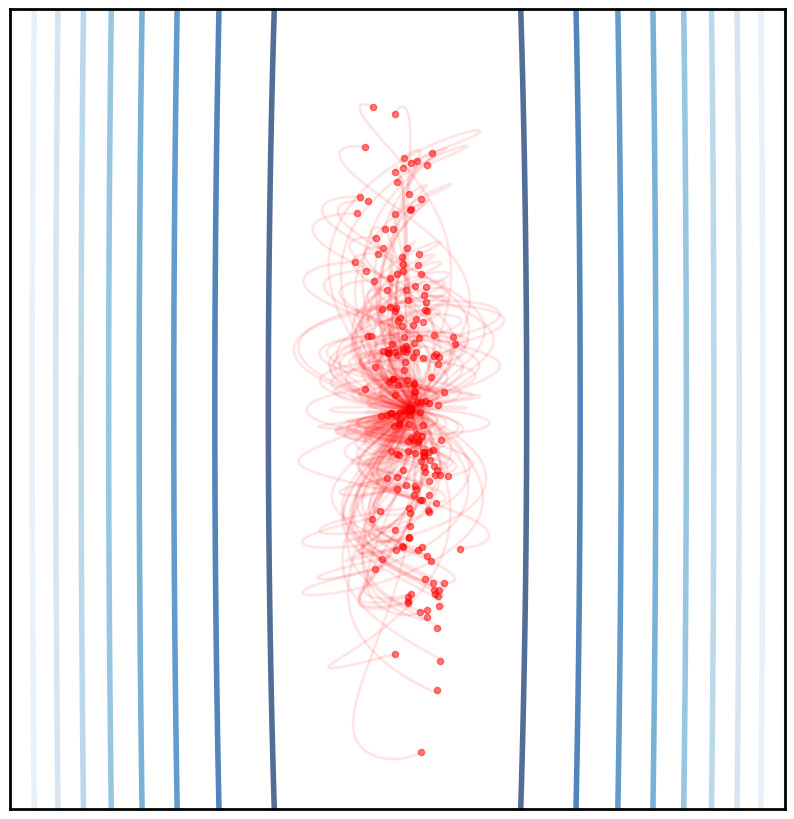}}%
		\subfigure[MHMC (g=1.0)]{
			\includegraphics[keepaspectratio, width=0.25\textwidth]{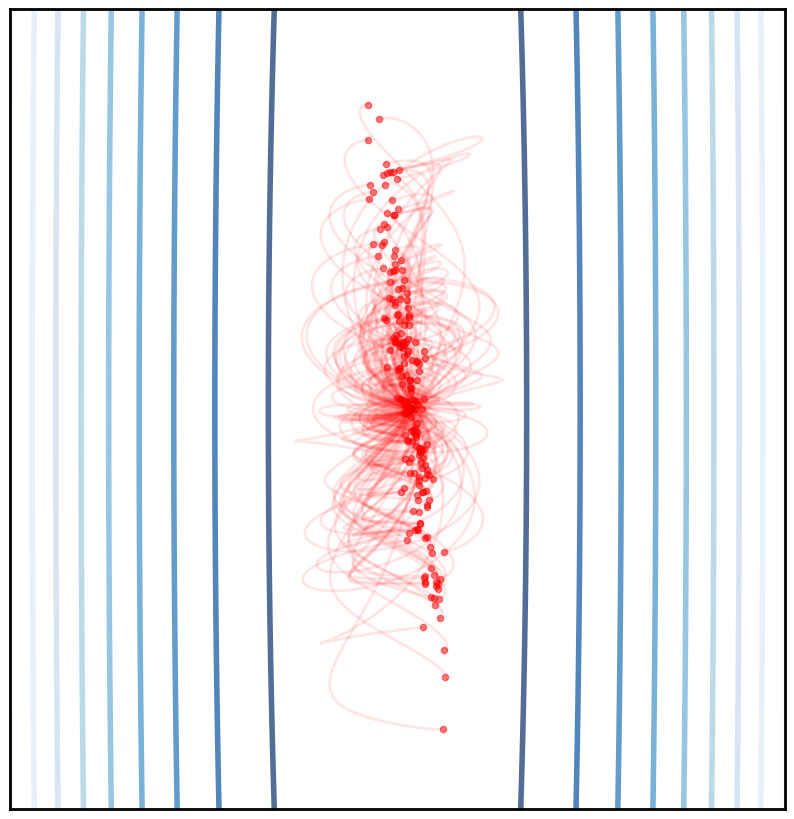}}%

		\subfigure[MHMC (g=2.0)]{
			\includegraphics[keepaspectratio, width=0.25\textwidth]{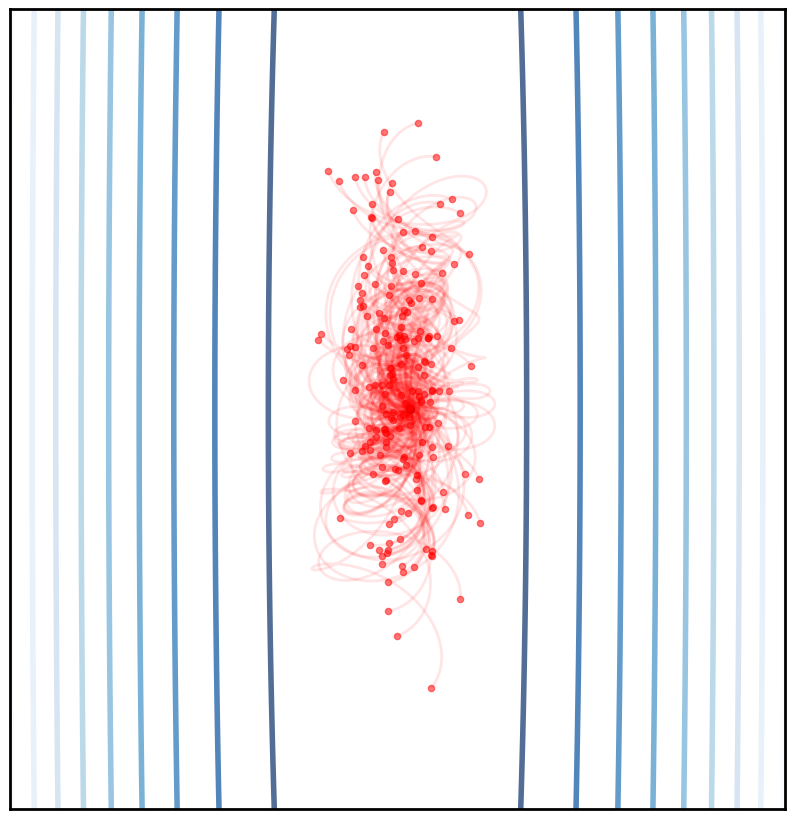}}%
		\subfigure[MHMC (g=3.0)]{
			\includegraphics[keepaspectratio, width=0.25\textwidth]{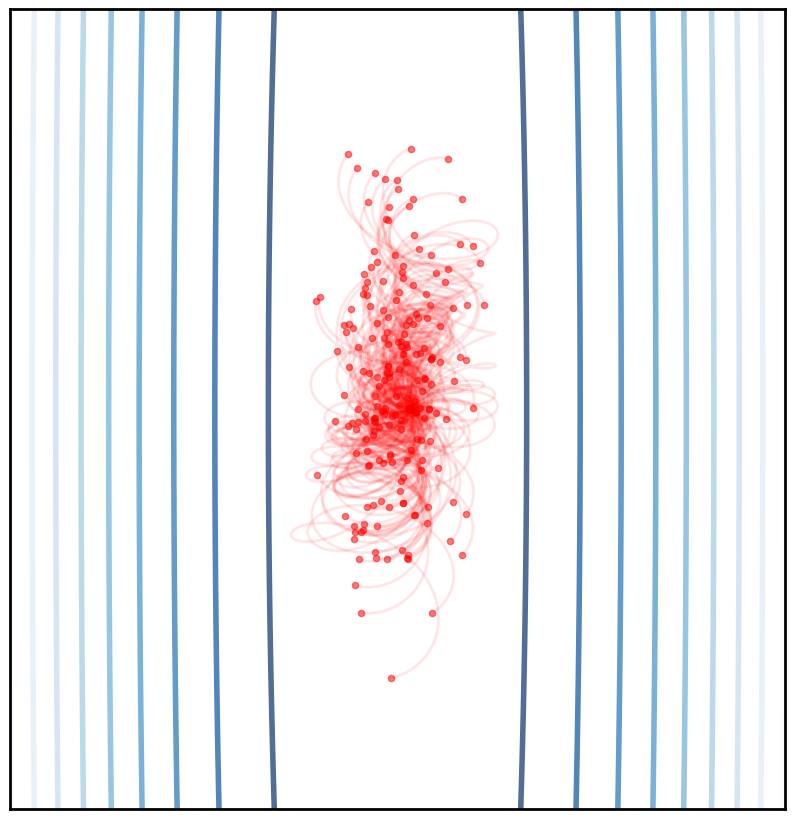}}%
		\subfigure[MHMC (g=4.0)]{
			\includegraphics[keepaspectratio, width=0.25\textwidth]{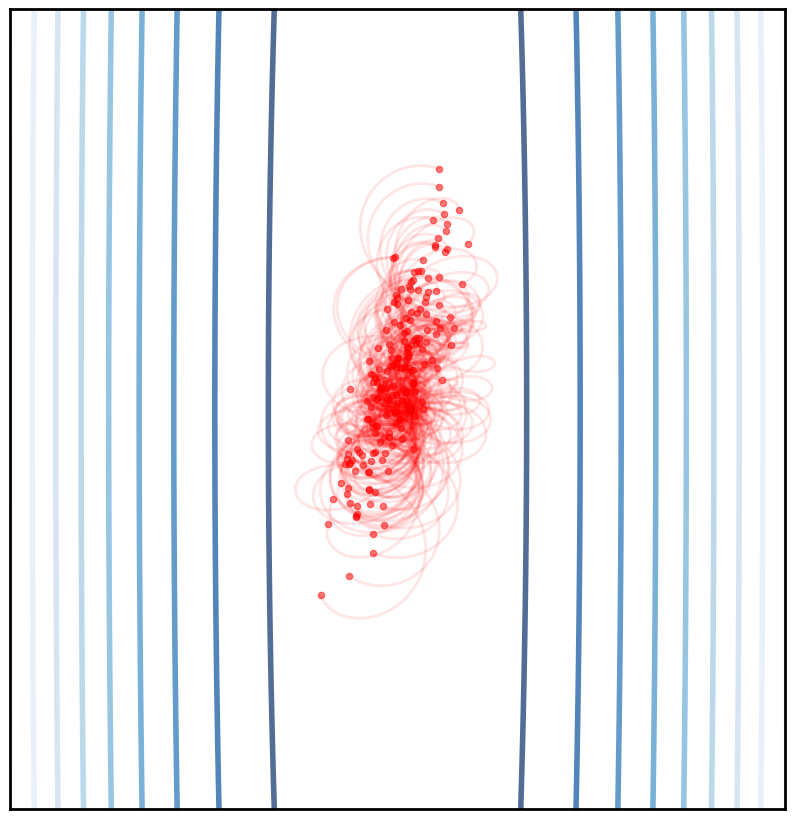}}%
		\caption{HMC and MHMC proposals for an anisotropic Gaussian target.}
		\label{fig:aniso-gaussian-props}
	\end{figure}

	\subsection{Banana Density}
	We also provide proposal illustrations for the banana density of \cite{Haario1999}, as shown in Figure \ref{fig:banana-props}.

	\begin{figure}[!ht]
		\centering
		\subfigure[Standard HMC (g=0)]{
			\includegraphics[keepaspectratio, width=0.31\textwidth]{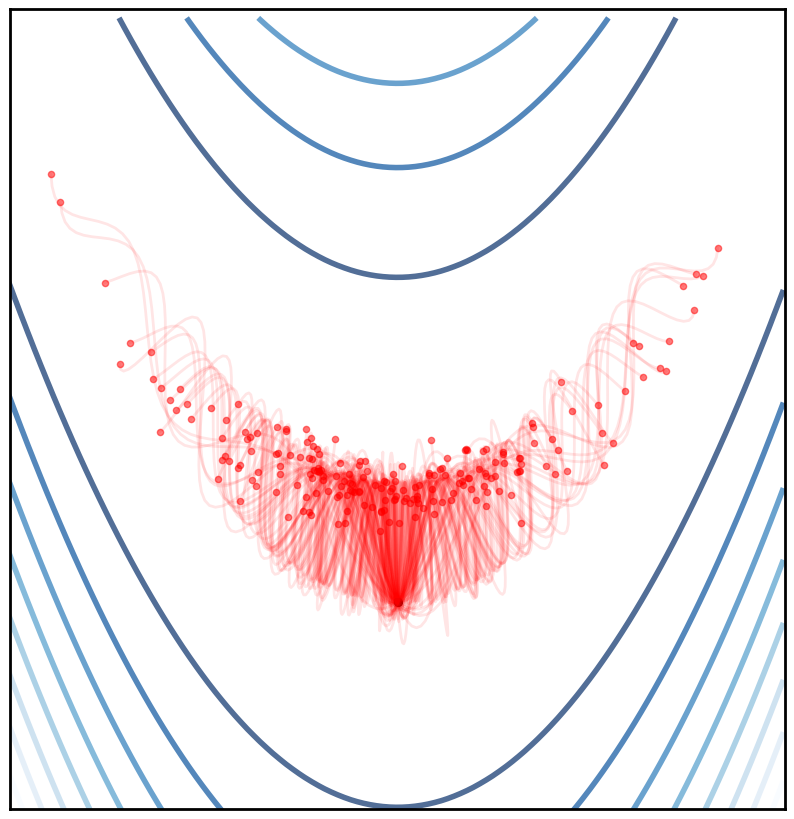}}%
		\subfigure[MHMC (g=0.1)]{
			\includegraphics[keepaspectratio, width=0.31\textwidth]{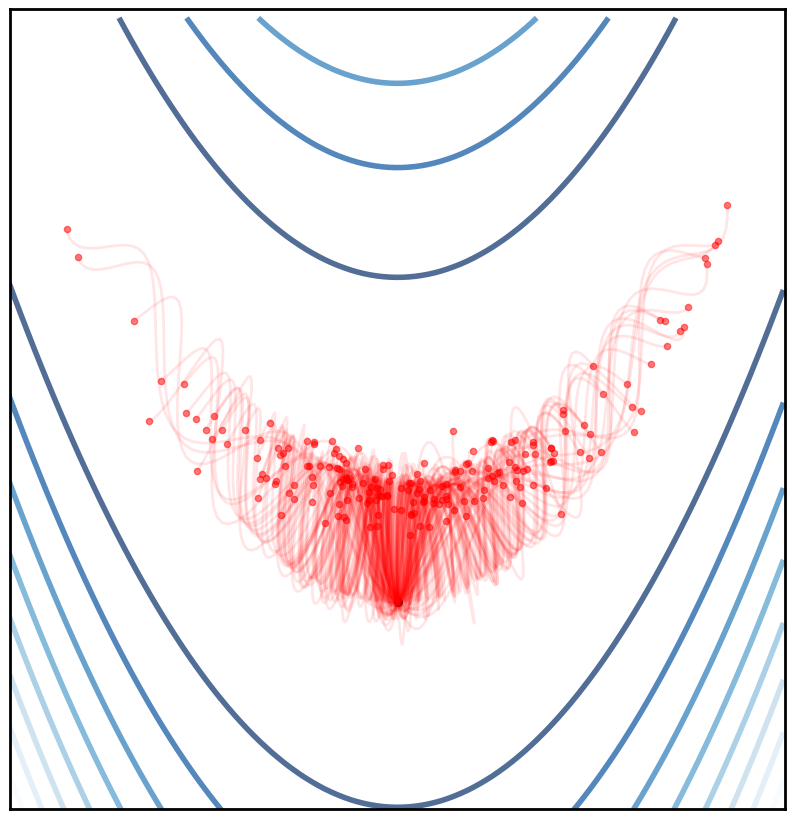}}%
		\subfigure[MHMC (g=0.2)]{
			\includegraphics[keepaspectratio, width=0.31\textwidth]{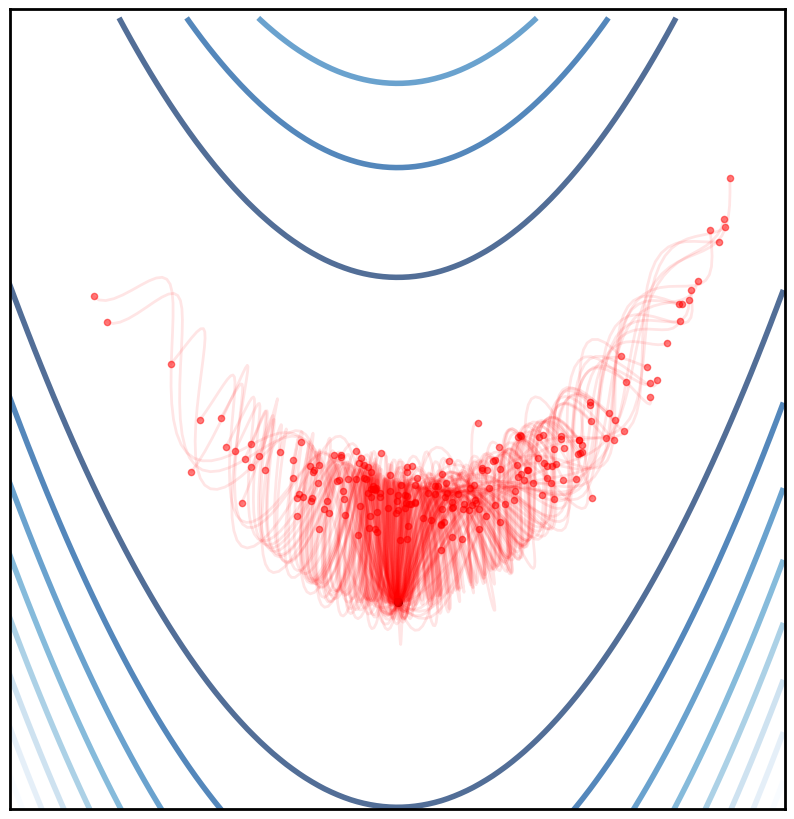}}%

		\subfigure[MHMC (g=0.3)]{
			\includegraphics[keepaspectratio, width=0.31\textwidth]{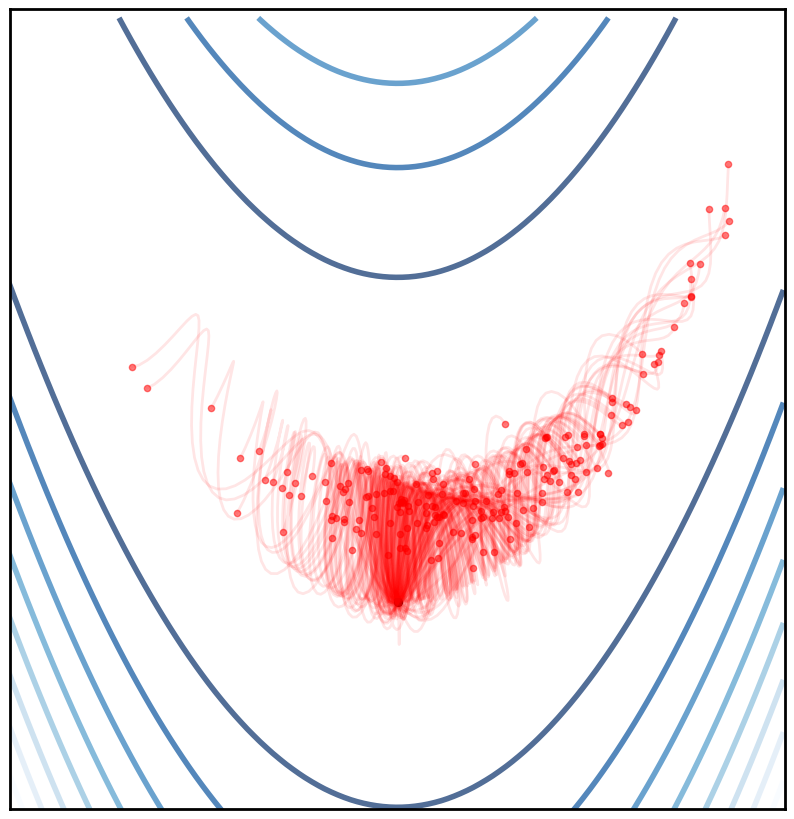}}%
		\subfigure[MHMC (g=0.4)]{
			\includegraphics[keepaspectratio, width=0.31\textwidth]{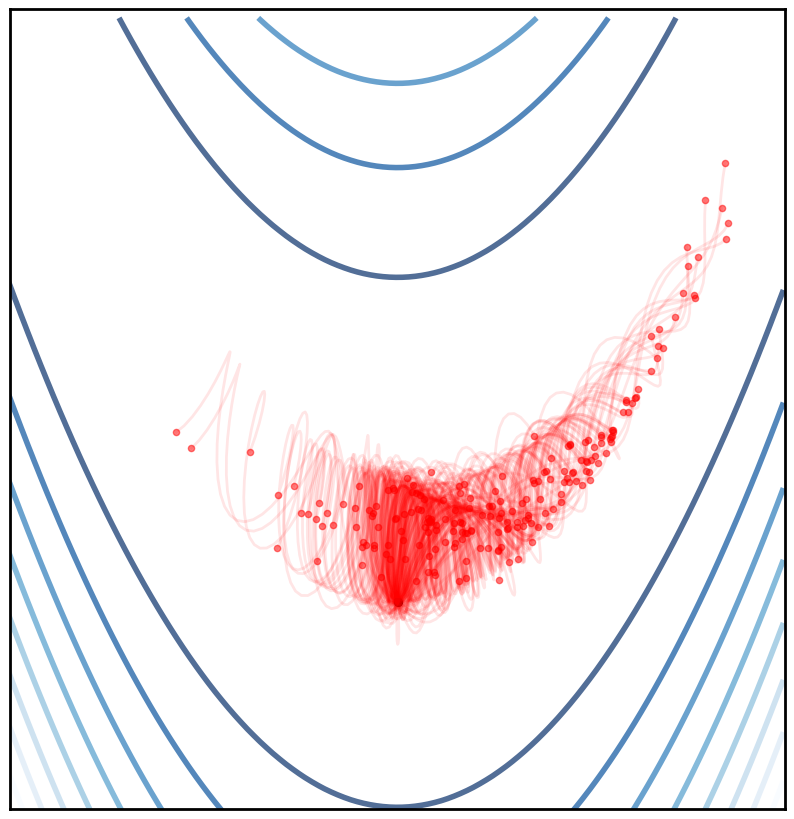}}%
		\subfigure[MHMC (g=0.5)]{
			\includegraphics[keepaspectratio, width=0.31\textwidth]{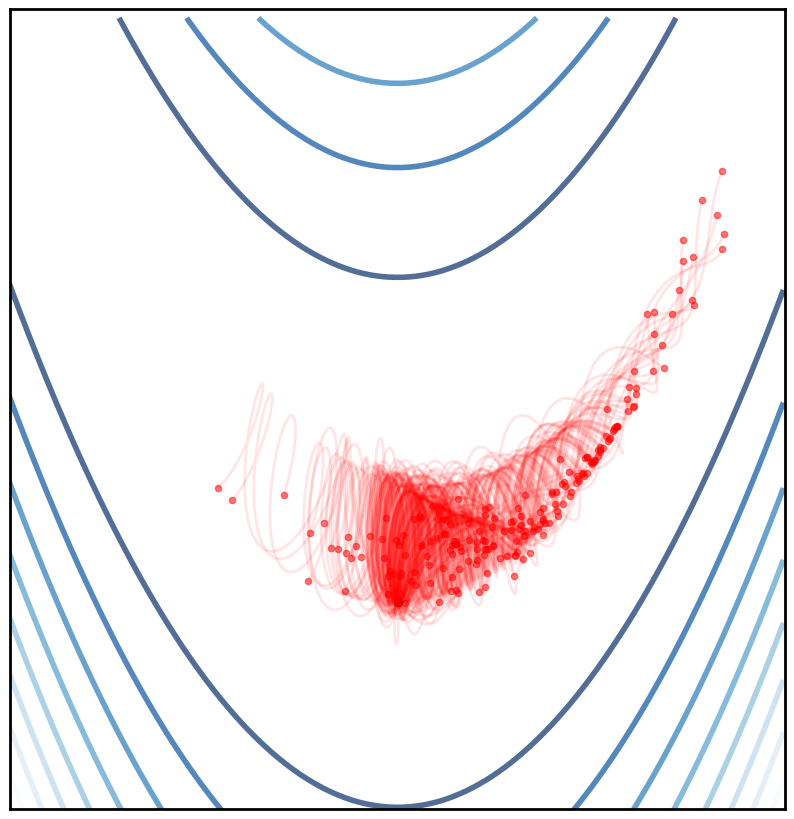}}%
		\caption{HMC and MHMC proposals for the banana density target.}
		\label{fig:banana-props}
	\end{figure}

	\clearpage
\end{document}